\documentclass[onefignum,onetabnum]{siamonline220329}

\usepackage{xr}
\makeatletter

\usepackage{todonotes}

\newcommand*{\addFileDependency}[1]{
\typeout{(#1)}
\@addtofilelist{#1}
\IfFileExists{#1}{}{\typeout{No file #1.}}
}\makeatother

\usepackage{lipsum}
\usepackage{amsfonts, amssymb}
\usepackage{graphicx,subfig}
\usepackage{epstopdf}
\usepackage{algorithmic}
\usepackage{wrapfig}
\ifpdf
  \DeclareGraphicsExtensions{.eps,.pdf,.png,.jpg}
\else
  \DeclareGraphicsExtensions{.eps}
\fi

\usepackage{colortbl,booktabs} 
\def\ROWCOLOR{black!15!white}
\usepackage{enumitem}
\usepackage{xcolor}
\definecolor{BLUE}{rgb}{0.3,0.3,0.9}
\definecolor{RED}{rgb}{0.8,0.05,0.05}
\definecolor{GREEN}{rgb}{0.05,0.5,0.05}

\usepackage{enumitem}
\setlist[enumerate]{leftmargin=.5in}
\setlist[itemize]{leftmargin=.5in}

\newcommand{\eg}{\textit{e.g.}\ }
\newcommand{\ie}{\textit{i.e.}\ }
\newcommand{\nb}{\textit{n.b.}\ }

\DeclareMathOperator*{\argmin}{argmin}

\newcommand{\pp}[2]{ \frac{\partial #1}{\partial #2} }
\newcommand{\dd}[2]{ \frac{\mathrm{d} #1}{\mathrm{d} #2} }
\newcommand{\itemsymbol}{{\small $\blacktriangleright$}}

\newcommand{\sA}{{\cal A}} 
\newcommand{\sB}{{\cal B}} 
\newcommand{\sC}{{\cal C}} 
\newcommand{\sD}{{\cal D}}

\newcommand{\sJ}{{\cal J}}
\newcommand{\sL}{{\cal L}}

\newcommand{\sN}{{\cal N}}
\newcommand{\sO}{{\cal O}}

\newcommand{\sV}{{\cal V}}

\newcommand{\sX}{{\cal X}}

\newcommand{\bbN}{{\mathbb N}} 
 
\newcommand{\bbR}{{\mathbb R}}
\newcommand{\bbE}{{\mathbb E}}

\newcommand{\VI}{$\mathrm{VI}(F(\cdot\ ;d),\sC)$}
\newcommand{\VITheta}{$\mathrm{VI}(F_{\Theta}(\cdot\ ;d),\sC)$}

\usepackage{enumitem}
\setlist[enumerate]{leftmargin=.5in}
\setlist[itemize]{leftmargin=.5in}

\newsiamremark{remark}{Remark}
\newsiamremark{hypothesis}{Hypothesis}
\crefname{hypothesis}{Hypothesis}{Hypotheses}
\newsiamthm{claim}{Claim}

\title{Operator Splitting for Learning to Predict   Equilibria in Convex Games}

\headers{Learning to Predict Equilibria}{McKenzie, \textit{et al.}}

\author{D. McKenzie\thanks{Department of Mathematics, Colorado School of Mines
  (\email{dmckenzie@mines.edu}).}
\and H. Heaton\thanks{Typal Academy
  (\href{https://research.typal.academy}{research.typal.academy}).}
\and Q. Li\thanks{Decision Intelligence Lab, DAMO Academy, Alibaba US.}
\and S. Wu Fung\footnotemark[2]
\and S. Osher\thanks{Department of Mathematics, UCLA.}
\and W. Yin\footnotemark[4]}

\usepackage{amsopn}

\newcommand{\edit}[1]{\textcolor{black}{#1}}
\newcommand{\reedit}[1]{\textcolor{black}{#1}}

\def\thmUniversalApproximation{
If Assumptions (A1)--(A5) hold, then, for all $\varepsilon > 0$, there exists  $F_{\Theta}(\cdot;\! \cdot)$ such that $\displaystyle \max_{d \in \sD}\|x_{d}^{\star} - \sN_{\Theta}(d)\|_2 \leq \varepsilon$. 
}

\def\thmDysConvergence{
Suppose $\sC^{1}$ and $\sC^2$ are as in \cref{thm:Use_DYS} and $F_{\Theta}$ is $\alpha$-cocoercive. If a sequence $\{z^k\}$ is generated via  $z^{k+1} = T_\Theta(z^k; d)$ for $T_\Theta$ in \eqref{eq: T-DYS-Param} with $\gamma = \alpha$ and $\{z : z = T_\Theta(z;d)\}\neq \varnothing$, then  $P_{\sC_1}(z^k)  \rightarrow x_d^\circ = \sN_{\Theta}(d)$.
\edit{Moreover, the computational complexity to obtain an estimate $x^k$ with fixed point residual norm no more than $\epsilon >0$ is $\sO\left( \mathrm{dim}(\sC)^2 / \epsilon^2 \right).$}
}

\ifpdf
\hypersetup{
  pdftitle={Learning to Predict Equilibria},
  pdfauthor={McKenzie}
}
\fi

\begin{document}

\maketitle

\begin{abstract}
Systems of competing agents can often be modeled as games. Assuming rationality, the most likely outcomes are given by an {\em equilibrium} (\eg a Nash equilibrium). In many practical settings,   games are influenced by   {\em context}, \ie additional data beyond the control of any agent (\eg weather for traffic and fiscal policy for market economies). Often the exact game mechanics are unknown, yet vast amounts of historical data consisting of (context, equilibrium) pairs are available, raising the possibility of {\em learning} a solver which predicts the equilibria {\em given only the context}. 
\edit{We introduce Nash Fixed Point Networks (N-FPNs), a class of neural networks that naturally output equilibria.}
Crucially, N-FPNs employ a constraint decoupling scheme to handle complicated agent action sets while avoiding expensive projections. \edit{Empirically, we find} N-FPNs are compatible with the recently developed Jacobian-Free Backpropagation technique 
for training implicit networks, making them significantly faster and easier to train than prior models. 
\edit{Our experiments show} N-FPNs are capable of scaling to problems orders of magnitude larger than existing learned game solvers.
\end{abstract}

\begin{keywords}
end-to-end learning, variational inequalities, game theory, operator splitting, machine learning.
\end{keywords}

\begin{MSCcodes}
68Q25, 68R10, 68U05
\end{MSCcodes}

\section{Introduction}
\begin{wrapfigure}[15]{r}{0.47\textwidth}
    \centering 
    \includegraphics[width=0.43\textwidth]{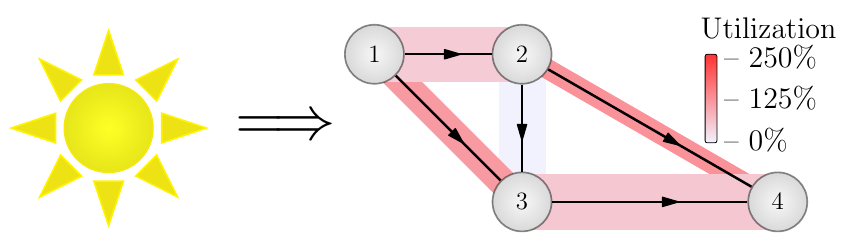} \\[-2pt]    
    \includegraphics[width=0.43\textwidth]{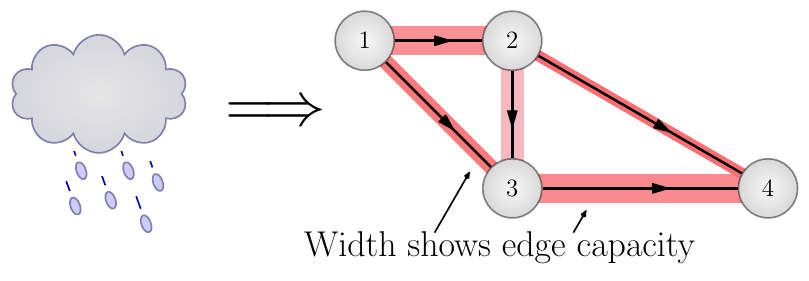}     
    \caption{Proposed N-FPNs can predict traffic flow (specifically, the utilization of each road segment) given only contextual information (\eg weather).}
    \label{fig: game-toy-diagram}
\end{wrapfigure}
Many recent works in deep learning   highlight  the power of using end-to-end learning in conjunction with known analytic models and constraints \cite{bertsimas2015data,romano2017little, de2018end,ling2018game,kotary2021end,gilton2021deep, chen2021learning, heaton2022wasserstein}. This best-of-both-worlds approach fuses the flexibility of learning-based approaches with the interpretability of models derived by domain experts. We further this line of research by proposing a practical framework for learning to predict the outcomes of contextual ({\em i.e.} parametrized) games from historical data  while respecting constraints on players' actions. 
Many social systems can aptly be analyzed as games, including market economies \cite{arrow1954existence}, traffic routing \cite{wardrop1952road}, even penalty kicks in soccer \cite{azar2011soccer}. We consider games with costs parametrized by a context variable $d$, beyond the control of any player. As in the multi-armed bandit literature, we call such games {\em contextual} \cite{sessa2020contextual}. For example, in traffic routing, $d$ may encode factors like   weather, local sporting events or tolls influencing drivers' commutes.

Game-theoretic analyses frequently assume players' cost functions are known {\em a priori} and seek to predict how  players will act, typically by computing a Nash equilibrium $x_d^{\star}$ \cite{nash1950equilibrium}. Informally, a Nash equilibrium is a choice of strategy for each player such that no player can improve their outcomes via unilateral deviation. 
However,  in practice the cost functions are frequently unknown. Here, we consider the problem of predicting equilibria, {\em given only contextual information}, without knowing players' cost functions. We do so by learning an approximation to the game gradient (see \eqref{eq: game-gradient}), so in this sense our work is closely related to the literature on inverse game theory. For technical reasons we focus on games in which each player's cost function is strongly convex. This property is sometimes referred to as ``diagonal strict convexity'' \cite{rosen1965existence}. We note this class of games include many routing games \cite{roughgarden2007routing}. Furthermore, in many cases it is possible---even desirable---to add a regularizer to each cost function such that it becomes strongly convex \cite{mertikopoulos2016learning}, see \cref{remark: QRE} for further discussion. 

For non-contextual games, many prior works (see  \Cref{section: related-works}) seek to use (noisy) observations of the Nash equilibrium to learn the cost functions. Our approach is different. We propose a new framework: Nash Fixed Point Networks (N-FPNs). Each N-FPN is trained on historical data pairs $(d, x_d^\star)$ to ``predict the appropriate game from context and then output game equilibria'' by \reedit{tuning an operator so that its} fixed points coincide with Nash equilibria. N-FPN inferences are computed by repeated application of the operator until a fixed point   condition is satisfied. Thus, by construction, N-FPNs are implicit networks---neural networks evaluated using an arbitrary number of layers~\cite{winston2020monotone,bai2019deep,ghaoui2019implicit,fung2022jfb}---and the operator weights can be efficiently trained using Jacobian-free backpropagation   \cite{fung2022jfb}. Importantly, the N-FPN architecture incorporates a constraint decoupling scheme derived from an application of Three-Operator splitting \cite{davis2017three}. This decoupling allows N-FPNs to avoid costly projections onto agents' action sets; the computational bottleneck of prior works \cite{ling2018game,ling2019large,li2020end}. This innovation allows N-FPNs to scale to large games, or to games with action sets significantly more complicated than the probability simplex.

One might enquire as to the {\em expressiveness} of N-FPNs. That is, can a given contextual game be arbitrarily well-approximated by a N-FPN? We answer this question in the affirmative, at least for contextual games possessing the diagonal strict convexity property alluded to above. 

Finally, we demonstrate how N-FPNs can be used to predict other closely related kinds of equilibria, particularly quantal response equilibria \cite{mckelvey1995quantal} and Wardrop equilibria \cite{wardrop1952road}. We complement our theoretical insights with numerical experiments demonstrating the efficacy and scalability of the N-FPN framework. We end by discussing how N-FPNs might be applied to other phenomena modeled by variational inequalities. \\

\paragraph{Contributions} 
We provide a \textit{scalable} data-driven framework for efficiently predicting equilibria in systems modeled as contextual games. Specifically, we do the following.
\begin{itemize}
    \item[\itemsymbol] Provide general, expressible, and end-to-end trained model predicting Nash equilibria.
    \item[\itemsymbol] Give  scheme for decoupling constraints for efficient forward and backward propagation.
    \item[\itemsymbol] Prove N-FPNs are universal approximators for a certain class of contextual games.
    \item [\itemsymbol] Demonstrate empirically the scalability of N-FPNs to large-scale problems.
    
\end{itemize}

\begin{table*}[t]
    \centering
    \small
    \begin{tabular}{c c c c c}
        Attribute &   Analytic & Feedforward & \cite{ li2020end},\cite{ling2018game} & Proposed N-FPNs\\\hline 
        \rowcolor{\ROWCOLOR}
        Output is Equilibria & \checkmark &   & \checkmark  & \checkmark \\ 
        Data-Driven & & \checkmark &  \checkmark & \checkmark \\
         \rowcolor{\ROWCOLOR}
         Constraint Decoupling &   \checkmark & &  & \checkmark   \\
 
        Simple Backprop & NA & \checkmark &  &  \checkmark  \\  
        
    \end{tabular}
    \caption{Comparison of different equilibria prediction methods. Analytic   modeling algorithms yield  game equilibria that are not data-driven. Traditional feed-forward networks are data-driven and easy to train, but are incapable of outputting a game equilibrium. Existing game-based implicit models are nontrivial to train (backpropagate) and require intricate forward propagation. }
    \label{tab: contributions}
\end{table*}

\section{Preliminaries}
\label{sec: preliminaries}


We begin with a brief review of relevant game theory. After establishing notation, we provide a set of assumptions under which the mapping $d \mapsto x_{d}^{\star}$ is ``well-behaved.'' 
We then describe variational inequalities  and   how Nash equilibria can be characterized using fixed point equations.  

\subsection{Games and Equilibria}
\label{sec: Games and Equilibria}
Let $\sX$ be a finite dimensional Hilbert space. A $K$-player normal form contextual game is defined by action sets\footnote{These are also known as   \textit{decision sets} and/or   \textit{strategy sets}.} $\sV_k$ 
and cost functions  $u_{k}: \sX\times \sD \to \bbR$ for $k\in [K]$, where the constraint profile is $\sC\triangleq \sV_1\times \ldots \times \sV_K$ and $\sD$ denotes the set of contexts (\ie data space). 
The $k$-th player's action   $x_k$ is constrained to the action set $\sV_k$, yielding an action profile  $x =\left({x}_{1}, \ldots, {x}_{K}\right) \in \sC \subseteq \sX$.
Actions of all players other than $k$ are 
${x}_{-k}=\left({x}_{1}, \ldots, {x}_{k-1}, {x}_{k+1}, \ldots, {x}_{K}\right)$. Each rational player aims to minimize their cost function $u_{k}$ by controlling only ${x}_{k}$ while explicitly knowing  $u_{k}$ is impacted by   other players' actions ${x}_{-k}$. An action profile $x_d^\star$ is a {\em Nash equilibrium} (NE)  provided, for all $x_k \in \sV_k$  and  $k \in [K]$, 
\begin{equation}
 u_{k}(x_{k},x_{d,-k}^{\star};d) \geq u_{k}(x_{d,k}^{\star},x_{d,-k}^{\star};d). 
\end{equation}
In words, $x_d^{\star}$ is a Nash equilibrium if no player can decrease their cost by unilaterally deviating from $x_d^{\star}$. Throughout, we make the following assumptions:

\begin{itemize}
    \item[(A1)] $\sC\subset \sX$ is closed and convex.
    \item[(A2)] The cost functions $u_{k}(x; d)$ are continuously differentiable with respect to $x$. 
    \item[(A3)] For all $x$, each $\nabla_ku_k(x;\cdot )$ is Lipschitz.
    \item[(A4)] Each cost function $u_k(x_k,x_{-k};d)$ is $\alpha$-strongly convex with respect to $x_k$. 
    \item[(A5)] The set of contextual data $\sD$ is compact. 
\end{itemize}

When the Assumption (A2) holds, we define the {\em game gradient} by
\begin{equation} 
    F(x;d) \triangleq \left[\nabla_{x_1}u_1(x;d)^\top \cdots\  \nabla_{x_K}u_K(x;d)^\top\right]^\top.
    \label{eq: game-gradient}
\end{equation}

\subsection{Variational Inequalities}
This subsection briefly outlines variational inequalities and their connection to games.

\begin{definition}
    For $\alpha > 0$ a mapping $F\colon\sX\times \sD \rightarrow\sX$ is \textit{$\alpha$-cocoercive}\footnote{This is also known as $\alpha$-inverse strongly monotone}  if, 
    \begin{equation}
        \left< F(x;d) - F(y;d), x-y\right> \geq \alpha \|F(x;d)-F(y;d)\|^2 \text{ for all } x,y\in\sX, d \in \sD
        \label{eq: def-F-cocoercive}
    \end{equation}
    \reedit{and $\alpha$-strongly monotone if 
    \begin{equation}
        \left< F(x;d) - F(y;d), x-y\right> \geq \alpha \|x - y\|^2 \text{ for all } x,y\in\sX, d \in \sD.
        \label{eq: def-F-monotone}
    \end{equation}}
    \noindent If \eqref{eq: def-F-monotone} holds for $\alpha=0$, then $F(\cdot\ ;d)$ is   \textit{monotone}.
\end{definition}

  
\begin{definition}
Given $d \in \sD$, a point $x_d^\star\in\sC$ is   a \textit{variational inequality \eqref{eq: VI-def} solution} provided
\begin{align}
\left<F(x_d^\star;d), x - x_d^\star\right> \geq 0,\quad \mbox{for all \ } x\in\sC.    
\tag{VI}
\label{eq: VI-def}
\end{align}
The solution set for \eqref{eq: VI-def} is denoted by \VI.
\end{definition}

 Nash equilibria may be characterized using VIs \cite{facchinei2007finite}; namely, 
\begin{equation}
    x_d^{\star} \text{ is an NE} \iff x_d^{\star} \in \mbox{\VI}.
    \label{eq: VI-equals-Nash}
\end{equation}
That is, $x_d^\star$ is an NE if no unilateral change improves any single cost and a VI solution if no feasible update improves the sum of all costs. By~\eqref{eq: VI-equals-Nash},  these   views are equivalent.

\subsection{Implicit Neural Networks}
Commonplace feedforward neural networks are a composition of  parametrized functions $T^{\ell}_{\Theta_{\ell}}(\cdot)$ (called layers) which take    data $d$ as input and return a prediction $y$. Formally, given $d$, a  network $\mathcal{N}_\Theta$ computes each inference  $y$ via 
\begin{equation}
\begin{split}
     & y = \mathcal{N}_\Theta(x) = x^{L+1},\\ 
\text{ where } & x^1 =d \ \mbox{and}\ x^{\ell+1} = T^{\ell}_{\Theta_{\ell}}(x^{\ell})
\ \ \mbox{for all $\ell \in [L]$.}
\end{split}
\end{equation} 
Instead of an explicit cascade of distinct compositions, {\em implicit neural networks} $\sN_\Theta$  
use a single mapping $T_\Theta$, and the output $\sN_\Theta(d)$ is defined {\em implicitly}\footnote{We reserve the notation $x^\star_d$ for denoting equilibria, fixed points, or VI solutions associated to the true game we wish to approximate. We use $x^\circ_d$ for denoting equilibria, fixed points, or VI solutions associated to the approximating neural network.} by an equation, \eg 
\begin{align}
     \sN_\Theta(d) \triangleq x^{\circ}_d  
    \ \     \text{where } \ \  x^{\circ}_{d} = T_{\Theta}(x^{\circ}_d;d). \label{eq:FP}  
\end{align}
Equation \eqref{eq:FP} can be solved via a number of methods, {\em e.g.}   fixed point iteration: $x^{k+1} = T_{\Theta}(x^{k};d)$. Implicit neural networks   recently received much attention as they admit a memory efficient backprop \cite{bai2019deep,bai2020multiscale,ghaoui2019implicit,fung2022jfb, geng2021training}. By construction, the output of $\mathcal{N}_{\Theta}(d)$ is a fixed point. Thus, several recent works explore using implicit networks in supervised learning problems where the target to be predicted can naturally be interpreted as a fixed point \cite{heaton2021feasibility,gilton2021deep,heaton2021learning,heaton2022explainable,mckenzie2023faster}. 

\section{Well-behaved equilibria}
We verify that assumptions (A1)--(A5) are sufficient to guarantee that $x_d^{\star}$ depends smoothly on $d$. This is crucial for showing that an N-FPN can approximate the relationship between $d$ and $x^\star_d$ (see \Cref{thm: Universal Approx}).

\begin{theorem}
If Assumptions (A1) to (A5) hold, then
\begin{enumerate}[label=(\arabic*)]
    \item there is a unique Nash Equilibrium $x_d^{\star}$\  for all $d \in \sD$;
    \item the map $d \mapsto x_{d}^{\star}$ is Lipschitz continuous.
\end{enumerate}
\label{thm:Lipschitz_Equib}
\end{theorem}

\begin{proof}
 Assumption (A4) implies the game gradient $F(\cdot\!; d)$ is $\alpha$-strongly monotone. Hence, by \cite[Theorem 2]{rosen1965existence} the Nash equilibrium $x_d^{\star}$ is unique; see also \cite[Theorem 2.2.3]{facchinei2007finite}. This proves part 1. For part 2, first observe that (A3) implies $F$ is Lipschitz continuous with respect to $d$ in addition to being $\alpha$-strongly monotone. \cite[Theorem 2.1]{dafermos1988sensitivity} then shows that around any fixed $\bar{d}\in\sD$ the map $d \mapsto x_d^{\star}$ is locally Lipschitz, {\em i.e.} there exists a constant $L_{\bar{d}}$ and an open neighborhood $N_{\bar{d}} \subset \sD$ of $\bar{d}$ upon which $d\mapsto x_d^{\star}$ is $L_{\bar{d}}$-Lipschitz continuous. As $\mathcal{D}$ is compact (Assumption (A5)) a standard covering argument converts this local Lipschitz property to a global Lipschitz property.
\end{proof}

\begin{remark}
\label{remark: QRE}
\edit{
Assumption (A4) is fairly restrictive, but is in line with prior work \cite{ling2018game, li2020end,allen2021using,bertsimas2015data,ratliff2014social,zhang2017data}.
For games where the $u_k$ are not strongly convex, one can add a regularizer: $\tilde{u}_k(x) = u_k(x) + h(x)$. As an illustrative example, consider the case where each player's action set is the probability simplex
\begin{equation}
   \sV_k = \Delta_n := \{x \in \mathbb{R}^n: \ \sum_{j}x_{k,j} = 1 \text{ and } x_{k,j} \geq 0\}. 
\end{equation}
Adding an entropic regularizer---{\em i.e.} $h(x) = \sum_jx_{k,j}\log(x_{k,j})$ as in \cite{ling2018game}---affords an elegant interpretation of the Nash Equilibrium of the resulting regularized game as the {\em Quantal Response Equilibrium} (QRE) \cite{mckelvey1995quantal} of the original game. QRE are a useful solution concept for boundedly rational agents ({\em e.g.} humans). They describe situations where agents are likely to select the best action, but may also select a sub-optimal action with non-zero probability. However, choosing such an $h$ means  $\|\nabla_k\tilde{u}_k\| \to \infty$ as $x_k$ approaches the boundary of $\sV_k$, which may be undesirable. Formally, one may resolve this by using a ``smoothed'' entropic regularizer $h(x) = \sum_jx_{k,j}\log(x_{k,j} + \epsilon)$, which does satisfy Assumption (A3), as discussed in \cite{liu2022inducing}. In practice this seems unnecessary; see Section~\ref{sec:Experiment_RPS}. Alternatively, one may use an $h(x)$ which does not diverge as $x_k$ approaches the boundary of $\sV_k$, such as a quadratic penalty. We refer the reader to \cite{mertikopoulos2016learning} for further discussion on the choice of $h(x)$ and the interpretation thereof.}
\end{remark}

\section{Proposed Method: Nash-FPNs} 
\label{sec: proposed_method}
\edit{Recall that our goal is to train a predictor capable of approximating $x^{\star}_d$ given only $d$. We assume a fixed, yet unknown, contextual game which induces a probability distribution $\mu$ on $\sD\times \sX$ relating $d$ and $x_{d}^{\star}$. As predictor we propose to use a {\em Nash Fixed Point Network} (N-FPN) $\sN_\Theta$, defined abstractly as the solution to a parametrized variational inequality:
\begin{equation}
    \sN_\Theta(d) \triangleq \mathrm{VI}(F_\Theta(\cdot;d), \sC).
    \label{eq:Def_as_VI}
\end{equation}
Below we discuss how a N-FPN can be viewed concretely as an implicit neural network.}
In our context, the set $\sC$ is a product of action sets $\sV_k$ and $F_{\Theta}(\cdot;\! \cdot)$ is a neural network with weights $\Theta$.

\edit{Fixing a smooth loss function $\ell:\sX\times\sX\rightarrow \bbR$, in principle one selects a predictor (\ie a choice of weights $\Theta$) via minimizing the population risk:
\begin{equation}
    \min_{\Theta} \bbE_{(d,x^{\star}_d)\sim\mu}\left[\ell( \sN_\Theta(d), x_d^\star)\right]
    \label{eq:Population_Risk}
\end{equation}
In practice one minimizes the {\em empirical risk}, given a training data set $\{(d_i,x^{\star}_{d_i})\}_{i=1}^N\sim_{\mu} \sD\times\sX$, instead \cite{vapnik1999nature}:
\begin{equation}
    \min_{\Theta}\sum_{i=1}^N \ell(\mathcal{N}_\Theta(d_i), x_{d_i}^\star)
    \label{eq:Empirical_Risk}
\end{equation}}
A similar approach was proposed in \cite{li2020end}; we discuss how our approach improves upon theirs in  \Cref{section: related-works}. First, we provide a novel theorem guaranteeing the proposed design has sufficient capacity to accurately approximate the mapping $d \mapsto x_{d}^{\star}$ for   games of interest.

\begin{theorem}[\sc Universal Approximation]
\thmUniversalApproximation
\label{thm: Universal Approx}
\end{theorem}
A proof of \Cref{thm: Universal Approx} can be found in the supplemental material.
\noindent Since N-FPNs are universal approximators in theory, two practical questions arise: 
\begin{enumerate}[label=\arabic*)]
    \item For a given $d$, how are inferences of $\sN_{\Theta}(d)$ computed?
    \item  How are weights $\Theta$ tuned using training data $\{d, x_d^{\star}\}$?
\end{enumerate}
We address each inquiry in turn.
As is well-known \cite{facchinei2007finite}, for all $\alpha > 0$, 
\begin{equation}
\begin{aligned}
     x_d^\circ \in \mbox{\VITheta}
    \quad\iff\quad 
    x_d^\circ = P_{\sC}(x_d^\circ \!-\! \alpha F_{\Theta}(x_d^\circ;d)),
    \label{eq: VI-Landweber}
    \end{aligned}
\end{equation}
where $P_{\sC}$ denotes the projection onto the set $\sC$, \ie $P_{\sC}(x) \triangleq \argmin_{y\in \sC} \|y-x\|^2.$
When the   operator $P_{\sC}\circ (\mathrm{I} - \alpha F_\Theta)$ on the right hand side of \eqref{eq: VI-Landweber} is tractable and well-behaved, inferences of $\sN_\Theta(d)$ can be computed via a fixed point iteration, as in \cite{li2020end}. Unfortunately, for some $\mathcal{C}$ computing $P_{\sC}$ and $\mathrm{d} P_{\sC}/{\mathrm{d}z}$ requires a number of operations scaling {\em cubicly} with the dimension of $\sC$ \cite{amos2017optnet}, rendering this approach intractable   even for moderately sized problems.

Our key insight is that {\em there are multiple ways  to turn \eqref{eq:Def_as_VI} into a fixed point problem}. Specifically, we propose a fixed point formulation which, while superficially more complicated, avoids expensive projections {\em and} is easy to backpropagate through. The key ingredient (see  \Cref{thm:Use_DYS}) is an application of three operator splitting \cite{davis2017three} which replaces $P_{\sC}$ with projection operators possessing simple and explicit projection formulas. Similar ideas can be found in \cite{davis2017three,pedregosa2018adaptive}, but to the best of our knowledge, this splitting has not yet appeared in the implicit neural network literature. 

\begin{algorithm}[t]
\caption{ Nash Fixed Point Network\  (Abstract Form)}
\label{alg: N-FPN-abstract}
\begin{algorithmic}[1]           
    \STATE{\begin{tabular}{p{0.52\textwidth}r}
     \hspace*{-8pt}  $\sN_\Theta(d):$
     & 
     \end{tabular}}    
    
    \STATE{\begin{tabular}{p{0.52\textwidth}r}
     \hspace*{2pt} $z^1 \leftarrow \tilde{z}$, \ $z^0\leftarrow \tilde{z}$, \ $n \leftarrow 1$
     & 
     \end{tabular}}

    \STATE{\begin{tabular}{p{0.52\textwidth}r}
     \hspace*{2pt} {\bf while}   $\|z^{n} - z^{n-1}\| > \varepsilon$ {\bf or} $n=1$
     & 
     \end{tabular}}  
     
    \STATE{\begin{tabular}{p{0.52\textwidth}r}
     \hspace*{10pt} $x^{n+1} \leftarrow P_{\sC^1}(z^n)$
     & 
     \end{tabular}}      

    \STATE{\begin{tabular}{p{0.52\textwidth}r}
     \hspace*{10pt} $y^{n+1} \leftarrow P_{\sC^2}(2x^{n+1}-z^n- \gamma F_\Theta(x^{n+1};d))$
     & 
     \end{tabular}}      
     
    \STATE{\begin{tabular}{p{0.52\textwidth}r}
     \hspace*{10pt} $z^{n+1} \leftarrow z^n -x^{n+1} + y^{n+1}$
     & 
     \end{tabular}}      
     
    \STATE{\begin{tabular}{p{0.52\textwidth}r}
     \hspace*{10pt} $n\leftarrow n+1$
     & 
     \end{tabular}}       

    \STATE{\begin{tabular}{p{0.52\textwidth}r}
     \hspace*{2pt} {\bf return} $P_{\sC^1}(z^n)$
     & 
     \end{tabular}}   
\end{algorithmic}
\end{algorithm}

We present this architecture concretely as  \Cref{alg: N-FPN-abstract}. With a slight abuse of terminology, we refer to this architecture also as an N-FPN. Although we find  \Cref{alg: N-FPN-abstract} to be most practical, we note other operator-based methods  (\eg ADMM and PDHG) can be used within the N-FPN framework via equivalences of different fixed point formulations of the VI.

The proposed fixed point operator $T_\Theta$ below in \eqref{eq: T-DYS-Param} is computationally cheaper to evaluate than that in \eqref{eq: VI-Landweber} when the projections $P_{\sC_1}$ and $P_{\sC_2}$ are computationally cheaper than $P_{\sC}$. For example, suppose $\sC$ is a polytope written in general form: $\sC = \{x: Ax = b \text{ and } x \geq 0\}$. Here, computing $P_{\sC}(x)$ amounts to solving the quadratic program $\min_{y\in\sC}\|x-y\|_2^2$. However, we may instead take $\sC_1 = \{x: Ax = b\}$ and $\sC_2 = \{x: x \geq 0\}$, both of which can\footnote{This depends on some properties of $A$ (\eg rank).} enjoy straightforward closed-form projection operators $P_{\sC_1}$ and $P_{\sC_2}$. Also, taking $\sC_2 = \sC$ and $\sC_1 = \sX$ (\ie the whole space) reduces \eqref{eq: VI-DYS} to \eqref{eq: VI-Landweber}. For completeness, we present this special case of N-FPN as  \Cref{alg: N-FPN-Simple}, as this is more comparable to the approaches proposed in prior work \cite{ling2018game, li2020end}. 

 Below we provide a lemma justifying the decoupling of constraints in the action set $\sC$. 
 Here we make use of polyhedral sets\footnote{A set is polyhedral if it is of the form $\{ x : \left<x,a^i\right> \leq b_i, \ \mbox{for \ } i \in [p]\}$, for $p\in\bbN$.}; however, this result also holds in a more general setting utilizing relative interiors of $\sC^1$ and $\sC^2$. By $\delta_\sC\colon \sX\rightarrow {\bbR}\cup \{+\infty\}$ we denote the indicator function defined such that $\delta_{\sC}(x) = 0$ in $\sC$ and $+\infty$ elsewhere. The subgradient of the indicator function (also known as the normal cone of $\sC$) is denoted by $\partial \delta_{\sC}$.
\begin{lemma}
\reedit{Fix $\gamma > 0$.} Suppose $\sC = \sC_1\cap \sC_2$ for convex $\sC_1$ and $\sC_2$.
If both $\sC_i$ are polyhedral or have relative interiors with a point in common and the VI has a unique solution, then
\begin{equation} 
     T_\Theta(x;d)
    \triangleq x \!-\! P_{\sC^1}(x) + P_{\sC^2}\left( 2P_{\sC^1}(x) \!-\! x \!-\! \gamma F_\Theta (P_{\sC^1}(x);d)\right)
    \label{eq: T-DYS-Param} 
\end{equation}
yields the equivalence 
\begin{equation}
    \sN_\Theta(d) = 
       x_d^\circ
    \ \iff \ 
    x_d^\circ = P_{\sC^1}(z^{\circ}_d) \text{ where } z^{\circ}_d = T_{\Theta}(z^{\circ}_d;d).
    \label{eq: VI-DYS}
\end{equation}
\label{thm:Use_DYS}
\end{lemma}

\begin{proof}
    We begin with the well-known equivalence relation \cite{facchinei2007finite}:
    \begin{equation}
        x_d^\circ \in \mbox{\VITheta}
        \ \iff \ 
        0 \in F_{\Theta}(x_d^\circ\ ;d) + \partial \delta_{\sC}(x_d^\circ).
        \label{eq: VI-Inclusion-equal-appendix}
    \end{equation}
    Because $\sC^1$ and $\sC^2$ are either polyhedral sets or share a common relative interior, we may apply \cite[Theorem 23.8.1]{rockafellar1970convex} to assert
    \begin{equation}
        \partial \delta_\sC = \partial\delta_{\sC^1} + \partial \delta_{\sC^2}.
        \label{eq: C-split-constraint-appendix}
    \end{equation}
    Consider three maximal\footnote{A monotone operator $M$ is \textit{maximal}  if there is no other monotone operator  $S$ such that $\mbox{Gra}(M) \subset \mbox{Gra}(S)$ properly \cite{ryu2022large}. This is a technical assumption that holds for all cases of our interest.}  monotone operators $A$, $B$ and $C$, with $C$ single-valued.
For $\gamma > 0$, let $J_{\gamma A}$ and $R_{\gamma A}$ be the resolvent of $\gamma A$ and reflected resolvent of $\gamma A$, respectively, \ie 
\begin{equation}
    J_{\gamma A} \triangleq (\mathrm{I} + \gamma A)^{-1}
    \quad \mbox{and} \quad 
    R_{\gamma A} \triangleq 2 J_{\gamma A} -\mathrm{I}.
\end{equation}
In particular,   the resolvent of $ \partial \delta_{\sC^i}$ is precisely the projection operator $P_{\sC^i}$ \cite[Example 23.4]{bauschke2017convex}.
    Using three operator splitting  (\eg see \cite[Lemma 2.2]{davis2017three} and \cite{ryu2022large}), we obtain the equivalence 
    \begin{align}
         0 \in (A+B+C)(x)
        \ \iff \ 
        x = J_{\gamma B}(z), 
    \end{align}
    where 
    \begin{align}
        z =  z - J_{\gamma B}(z) + J_{\gamma A}(R_{\gamma B}-\gamma CJ_{\gamma B})(z). 
        \label{eq: DYS-Inclusion-Equivalence}
    \end{align} 
    Setting $A = \partial \delta_{\sC^2}$, $B = \partial \delta_{\sC^1}$, and $C = F_{\Theta}$, \eqref{eq: DYS-Inclusion-Equivalence} reduces to
    \begin{align}
        0 \in F_{\Theta}(x_d^\circ; d) + \partial \delta_{\sC^{1}}(x_d^\circ) + \partial \delta_{\sC^{2}}(x_d^\circ)     
        \quad \iff \quad
         x_d^\circ = P_{\sC_1}(z_d^\circ), 
          \ \mbox{where }           
        z_d^\circ = T_\Theta(z_d^\circ;d).
        \label{eq: DYS-inclusion-appendix}
    \end{align}
    Combining \eqref{eq: VI-Inclusion-equal-appendix}, \eqref{eq: C-split-constraint-appendix}, and \eqref{eq: DYS-inclusion-appendix} yields the desired result.
\end{proof}

\subsection{Forward propagation}
Given the operator $T_{\Theta}(\cdot  ; d)$, there are many algorithms for determining its fixed point $x_d^{\circ}$. Prior works \cite{ling2018game,li2020end} use Newton-style methods, which are fast for small-scale and sufficiently smooth problems. But they may scale poorly to high dimensions (\ie large $\mbox{dim}(\sC)$). We employ Krasnosel'ski\u{\i}-Mann (KM) iteration, which is the abstraction of splitting algorithms with low per-iteration computational and memory footprint. This is entirely analogous to the trade-off between first-order (\eg gradient descent, proximal-gradient) and second-order methods  (\eg  Newton) in high dimensional optimization; see \cite{ryu2022large} for further discussion. The next theorem provides a sufficient condition under which KM iteration converges. As the proof is standard we relegate it to the supplemental material.
\begin{theorem}
\thmDysConvergence
\label{thm: DYS_Convergence}
\end{theorem} 
We simplify the iterate updates for $T_\Theta$ in (\ref{eq: T-DYS-Param}) by introducing auxiliary sequences $\{x^k\}$ and $\{y^k\}$; see \Cref{alg: N-FPN-abstract}. 

\reedit{\begin{remark} There are several ways to design the architecture of $F_{\Theta}$ so that it is guaranteed to be cocoercive, regardless of $\Theta$. For example:
\begin{enumerate}
    \item One easily verifies that if $F_{\Theta}(\cdot;d)$ is $\alpha$-strongly monotone and $L$-Lipschitz then it is $\alpha/L^2$ cocoercive \cite{marcotte1995convergence}. Spectral normalization \cite{miyato2018spectral} can be used to ensure $F_{\Theta}(\cdot;d)$ is $1$-Lipschitz for most architectural choices. If a linear (in $x$) $F_{\Theta}(\cdot;d)$ suffices, one may use the parametrization suggested in \cite{winston2020monotone}:
    \begin{equation}
        F_{\Theta}(x;d) = \left(\alpha I + A^{\top}A + B^{\top} - B\right)x + Q_{\Theta^\prime}(d)
        \label{eq:N-FPN-CoCo}
    \end{equation}
    with $\alpha \in (0,1)$ to guarantee that $F_{\Theta}$ is $\alpha$-strongly monotone, in addition to spectral normalization. Here $Q_{\Theta^{\prime}}$ is any neural network mapping context to latent space and $A,B$ may depend on $d$. We implement this in \Cref{sec:Experiment_RPS} and observe it performs well. If a more sophisticated $F_{\Theta}$ is required, one could use \cite{pesquet2021learning} to parametrize a nonlinear monotone operator $\tilde{F}_{\Theta}$, whence $F_{\Theta} = \alpha I + \tilde{F}_{\Theta}$ is $\alpha$-strongly monotone. We caution that the parametrization given in \cite{pesquet2021learning} is indirect---$\tilde{F}_{\Theta}$ is given as the resolvent of a nonexpansive operator $Q_{\Theta}$---and so is unlikely to work well with three operator splitting. 
    \item By the Baillon-Haddad theorem \cite{baillon1977quelques,bauschke2010baillon} if $f_{\Theta}(x;d)$ is a convex and $L$-Lipschitz differentiable $\mathbb{R}$-valued function then $F_{\Theta}(x;d) = \nabla_x f_{\Theta}(x;d)$ is $1/L$ cocoercive. Using the architecture proposed in \cite{amos2017input} guarantees $f_{\Theta}(\cdot;d)$ is convex (in $x$), and spectral normalization may again be applied to $F_{\Theta}(x;d)$ to ensure $1$-Lipschitz differentiability. However, it is not clear that $F_{\Theta}(x;d)$ constructed in this manner will be easy to train. Indeed, \cite[Section 2]{salimans2021should} suggests it is better to parametrize $F_{\Theta}(x;d)$ directly, and not as the gradient of some function for an analogous problem in diffusion-based generative modeling  
\end{enumerate}
\end{remark}}

\begin{algorithm}[H]
\caption{ N-FPN -- Projected Gradient (Special Case)}
\label{alg: N-FPN-Simple}
\begin{algorithmic}[1]           
    \STATE{\begin{tabular}{p{0.533\textwidth}r}
     \hspace*{-8pt}  $\sN_\Theta(d):$
     & 
     $\vartriangleleft$ Input data is $d$
     \end{tabular}}    
    
    \STATE{\begin{tabular}{p{0.533\textwidth}r}
     \hspace*{-2pt} $x^1 \leftarrow \tilde{x}$, $n\leftarrow 2$,
     & 
     $\vartriangleleft$ Initializations
     \end{tabular}}        
     
    \STATE{\begin{tabular}{p{0.533\textwidth}r}
     \hspace*{-2pt} $x^2 \leftarrow P_{\sC}(x^1-F_{\Theta}(x^1;d))$
     & 
     $\vartriangleleft$ Apply $T$ update
     \end{tabular}}          

    \STATE{\begin{tabular}{p{0.533\textwidth}r}
     \hspace*{-2pt} {\bf while} $\|x^{n} - x^{n-1}\| > \varepsilon$ 
     & 
     $\vartriangleleft$ Loop to converge
     \end{tabular}}  
     
    \STATE{\begin{tabular}{p{0.533\textwidth}r}
      \hspace*{4pt} $x^{n+1} \!\leftarrow\! P_{\sC}(x^n- F_{\Theta}(x^n;d))$
     & 
     $\vartriangleleft$ Apply $T$ update
     \end{tabular}}      

    \STATE{\begin{tabular}{p{0.533\textwidth}r}
      \hspace*{4pt} $n\leftarrow n+1$
     & 
     $\vartriangleleft$ Iterate counter
     \end{tabular}}       

    \STATE{\begin{tabular}{p{0.533\textwidth}r}
     \hspace*{-2pt} {\bf return} $x^n$
     &   $\vartriangleleft$ Output inference 
     \end{tabular}}   
\end{algorithmic}
\end{algorithm}    

\subsection{Backpropagation}
\label{sec: Backpropagation}

\edit{In order to solve \eqref{eq:Empirical_Risk} using gradient based methods such as stochastic gradient descent or ADAM \cite{kingma2015adam} one needs to compute the gradient  $\mathrm{d}\ell/\mathrm{d}\Theta$.} To circumvent backpropagating through each forward step, $\mathrm{d}\ell/\mathrm{d}\Theta$ may be expressed by\footnote{All arguments are implicit and use   N-FPNs defined by \eqref{eq: VI-DYS}.}
\begin{equation}
    \label{eq:Jacobian_Equation}
    \dd{\ell}{\Theta} = \dd{\ell}{x} \dd{\sN_\Theta}{\Theta}
    = \dd{\ell}{x} \dd{P_{\sC^1}(z_d^\circ)}{z}\dd{z_d^\circ}{\Theta}.
\end{equation}

Starting with the definition of $z_d^\circ$ as a fixed point:
\begin{equation}
    z^{\circ}_d = T_{\Theta}(z^{\circ}_d;d)
\end{equation}
and appealing to the implicit function theorem \cite{krantz2012implicit} we obtain the Jacobian-based equation
\begin{equation}
    \label{eq: adjoint}
    \dd{z_d^\circ}{\Theta} = \sJ_\Theta^{-1} \pp{T_{\Theta}}{\Theta}, \quad     \mbox{with} \quad \sJ_\Theta \triangleq \mathrm{Id} - \dd{T_\Theta}{z}.
\end{equation}
Solving \eqref{eq: adjoint} (assuming $\sJ_{\Theta}$ is invertible, see \Cref{remark: Jacobian_invertibility}) is computationally intensive for large-scale games. Instead, we employ JFB, which consists of replacing $\sJ_\Theta^{-1}$ in~\eqref{eq: adjoint} with the identity matrix \edit{and using
\begin{equation}
    p_{\Theta} := \dd{\ell}{x}\dd{P_{\sC^1}(z_d^\circ)}{z}\pp{T_{\Theta}}{\Theta}
\end{equation}
in lieu of $\mathrm{d}\ell/\mathrm{d}\Theta$.} This substitution  yields a preconditioned gradient  and  is   effective for training in image classification~\cite{fung2022jfb} and data-driven CT reconstructions~\cite{heaton2021feasibility}.
Importantly, using JFB only  requires backpropagating through a \emph{single} application of $T_{\Theta}$ (\ie the final forward step) \edit{in order to compute $\partial T_{\Theta}/\partial \Theta$}.
\edit{
\begin{remark}
\label{remark: Jacobian_invertibility}
One sufficient condition commonly used to guarantee the invertibility of $\sJ_{\Theta}$ is to assume $T_{\Theta}$ is {\em contractive}, although this condition rarely holds in practice, and implicit networks empirically perform well without a firm guarantee of invertibility \cite{bai2019deep,bai2020multiscale}. Contractivity of $T_{\Theta}$ is also necessary to guarantee that $p_{\Theta}$ is a descent direction, see \cite[Theorem 3.1]{fung2022jfb}, although again this appears unnecessary in practice \cite{fung2022jfb,ramzi2022shine,koyama2022music,zhangmultiset}. We note that $T_{\Theta}$ is averaged if $F_{\Theta}$ is cocoercive. The use of JFB for averaged operators is an ongoing topic of interest, see \cite{mckenzie2023faster}.
\end{remark}}

    

\section{Further Constraint Decoupling}
\label{sec: Decoupling}
As discussed above, the architecture expressed in  \Cref{alg: N-FPN-abstract} provides a massive computational speed-up over prior architectures when $\sC = \sC_1 \cap \sC_2$ and $P_{\sC_1}$ and $P_{\sC_2}$ admit explicit and computationally cheap expressions, \eg when $\sC$ is a polytope. Yet, in many practical problems $\sC$ has a more complicated structure. For example, it may be the intersection of a large number of sets (\ie $\sC=\sC_1\cap\cdots\cap\sC_K$) or the Minkowski sum of intersections of simple sets (\ie $\sC = \sC_{1}+\cdots + \sC_{K}$ where $\sC_k = \sC_k^1\cap \sC_k^2$). We generalize our decoupling scheme by passing to a product space. With this extended decoupling we propose an N-FPN architecture with efficient forward propagation (\ie evaluation of $\sN_\Theta$) and backward propagation (to tune weights $\Theta$) using only the projection operators $P_{\sC_k}$ (for the $K$-intersection case) or $P_{\sC_{k}^i}$ (for the Minkowski sum case). We discuss the Minkowski sum case here, and defer the $K$-intersection case to the supplemental material.

\subsection{Minkowski Sum} 
\label{sec: Minkowski Sum}
This subsection provides a decoupling scheme for   constraints structured as a Minkowski sum,\footnote{This arises in the modeling Wardrop equilibria in traffic routing problems.} \ie 
\begin{equation}
    \sC \triangleq \sC_1 + \cdots + \sC_K,
\end{equation}
where $\sC_{k}\subset \sX$ and $\sC_{k} = \sC_{k}^{1}\cap \sC_{k}^{2}$ for all $k\in[K]$. 
The core idea is to avoid attempting to directly project onto $\sC$ and instead perform simple projections onto each set $\sC_k^i$, assuming the projection onto $\sC_k^i$ admits an explicit formula.  First, define the product space
\begin{equation}
\overline{\sX} \triangleq \underbrace{\sX\times \sX \times \ldots \times \sX}_{\text{$K$ times }}.
\end{equation}
For notational clarity, we denote elements of $\overline{\sX}$ by overlines so   each element $\overline{x}\in \overline{\sX}$ is of the form $\overline{x} = (\overline{x}_1,\ldots, \overline{x}_K)$ with $\overline{x}_{k} \in \sX$ for all $k\in[K]$. Because $\sX$ is a Hilbert space, $\overline{\sX}$ is naturally endowed with a scalar product $\left<\cdot,\cdot\right>_{\overline{\sX}}$ defined by
\begin{equation}
    \langle \overline{x}, \overline{y} \rangle_{\overline{\sX}} \triangleq \sum_{k=1}^{K} \langle \overline{x}_k,\overline{y}_k \rangle .
\end{equation}
Between $\sX$ and the product space $\overline{\sX}$ we define two maps $Q^-\colon\overline{\sX}\rightarrow \sX$ and $Q^+ :\sX\rightarrow\overline{\sX} $:
\begin{equation}
\begin{aligned}
    Q^-(\overline{x}) \triangleq  
    \sum_{k=1}^{K} \overline{x}_k, \quad \text{ and } \quad 
    Q^+(x) \triangleq (\underbrace{x,x,\ldots, x}_{\text{$K$ copies}}).
\end{aligned}    
\end{equation}
In words, $Q^-(\overline{x})$ maps down to $\sX$ by adding together the blocks of $\overline{x}$ and $Q^+(x)$ maps up to $\overline{\sX}$ by making $K$ copies of $x$, thus motivating the use of ``$+$'' and ``$-$'' signs.
Define the Cartesian product
\begin{equation}
  \sA \triangleq \sC_1\times  \ldots \times \sC_{K} \subseteq \overline{\sX},
\end{equation}
and note  $Q^- \left(\sA\right) = \sC$. 
To further decouple each set $\sC_k$, also define the Cartesian products
\begin{equation}
    \sA^{i} \triangleq \sC_{1}^{i}\times \ldots \times \sC_{K}^{i} 
    \quad  \mbox{for all $i\in[2]$.}
\end{equation}
so $\sA = \sA^{1}\cap \sA^{2}$. Note the projection onto $\sA^{i}$ can be computed component-wise; namely,
\begin{equation}
    P_{\sA^{i}}(\overline{x}) = \left(P_{\sC^{i}_1}(\overline{x}_1), \ldots, P_{\sC^{i}_K}(\overline{x}_K)\right) 
   \quad \mbox{for all $i\in[2]$.}
\end{equation}
We now rephrase  \Cref{alg: N-FPN-abstract}, applied to a VI in the product space $\mathrm{VI}\left( Q^+\circ F\circ Q^-, \sA\right)$, into  \Cref{alg: N-FPN-product} using $\sA^{i}$ in lieu of $\sC^{i}$. Here $F$ represents a neural network $F_{\Theta}(\cdot;\! d)$ with weights $\Theta$; for notational clarity we omit the arguments and subscript. The use of  \Cref{alg: N-FPN-product} is justified by the following two lemmas. The first shows the product space operator is monotone whenever $F$ is. The second shows the solution sets to the two VIs coincide, after applying $Q^{-}$ to map down from $\overline{\sX}$ to $\sX$.

 
\begin{lemma}\label{lemma: product-F-monotone}
    If $F: \sX \to \sX$ is $\alpha$-cocoercive, then      $Q^+\circ F \circ Q^-$ on $\overline{\sX}$ is $(\alpha/K)$-cocoercive. 
\end{lemma}
\begin{proof}
Fix any $\overline{x},\overline{y}\in\overline{\sX}$ and set   $R_{\overline{x}} \triangleq (F\circ Q^-)(\overline{x})$ and  $R_{\overline{y}} \triangleq (F\circ Q^-) (\overline{y})$. Then observe
\begin{subequations}
    \begin{align}
    \left< Q^+(R_{\overline{x}}) -  Q^+(R_{\overline{y}}), \overline{x} - \overline{y}\right>_{\overline{\sX}}
    &\; = \sum_{k=1}^K \left<  R_{\overline{x}} - R_{\overline{y}}, \overline{x}_k - \overline{y}_k\right> 
    \\
    &\; = \left<  R_{\overline{x}} - R_{\overline{y}},  Q^-(\overline{x}) - Q^-(\overline{y})\right> .
    \end{align}
\end{subequations}
Substituting in the definition of $R_{\overline{x}}$ and $R_{\overline{y}}$ reveals
\begin{subequations} 
\begin{align}
    \left< Q^+(R_{\overline{x}}) -  Q^+(R_{\overline{y}}), \overline{x} - \overline{y}\right>_{\overline{\sX}}
    \; & = \left< F(Q^-(\overline{x})) - F(Q^-(\overline{y})),  Q^-(\overline{x})- Q^-(\overline{y})\right>\\
    & \geq \alpha  \| F(Q^-(\overline{x})) - F(Q^-(\overline{y}))\|^2 \\
    & = \frac{\alpha}{K}\|Q^{+}\circ F\circ Q^-(\overline{x}) - Q^{+}\circ F\circ Q^-(\overline{y})\|_{\overline{\sX}}^2,
\end{align}\label{eq: F-product-space-monotone}\end{subequations}
where the final equality follows from the definition of the norm on $\overline{\sX}$.
Because \eqref{eq: F-product-space-monotone} holds for arbitrary  $\overline{x},\overline{y}\in\overline{\sX}$, the result follows.
\end{proof}

\begin{proposition}\label{lemma: minkowski-VI-equivalence}
    For  $F\colon \sX\rightarrow\sX$,
    $\overline{x}^\circ \in \mathrm{VI}\left( Q^+\circ F\circ Q^-, \sA\right)$ if and only if $   Q^-(\overline{x}^\circ) \in \mathrm{VI}\left(F,\sC\right)$.
\end{proposition}
\begin{proof}
Fix $\overline{y} \in \sA$ and  $\overline{x}^\circ \in  \mathrm{VI}\left( Q^+\circ F\circ Q^-, \sA\right)$. Similarly to the proof of  \Cref{lemma: product-F-monotone}, observe
\begin{subequations}
    \begin{align}
    \left< ( Q^+\circ F\circ Q^-)(\overline{x}^\circ), \overline{y}-\overline{x}^\circ\right>_{\overline{\sX}}
    \; &= \sum_{k=1}^K \left<   (F\circ Q^-)(\overline{x}^\circ), \overline{y}_k - \overline{x}_k^\circ\right> \\
    &= \left<  F(Q^-(\overline{x}^\circ)) , Q^-(\overline{y})-Q^-(\overline{x}^\circ) \right>.
    \end{align}
\end{subequations}
Because $Q^-(\sA) = \sC$, it follows that $  x^\circ \triangleq Q^-(\overline{x}^\circ) \in \sC$ and $w \triangleq Q^-(\overline{y}) \in \sC$.  Consequently,  
\begin{equation}
\begin{aligned}
    0 \leq \left< ( Q^+\circ F\circ Q^-)(\overline{x}^\circ), \overline{y}-\overline{x}^\circ\right>_{\overline{\sX}}
    \: = \left< F(x^\circ), w - x^\circ\right>.
 \end{aligned}  
    \label{eq: product-scalar-product-equiv}
\end{equation} 
Because $\overline{y}$ was arbitrarily chosen, (\ref{eq: product-scalar-product-equiv}) holds for all $w \in \sC$ and, thus, $Q^-(\overline{x}^\circ) \in \mathrm{VI}\left(F,\sC\right)$.\\

Conversely, fix $\overline{y}\in \sA$ and  $\overline{x}^\circ \in \overline{\sX}$ such that $Q^{-}(\overline{x}^\circ) \in \mathrm{VI}(F,\sC)$. Then $Q^-(\overline{y}) \in \sC$ and
\begin{subequations}
    \begin{align}
    0 & \leq \left<  F(Q^-(\overline{x}^\circ)) ,  Q^{-}(\overline{y}) - Q^-(\overline{x}^\circ) \right>\\
    & =    \sum_{k=1}^K \left<  F(Q^-(\overline{x}^\circ)) ,   \overline{y}_k - \overline{x}^\circ_k \right>\\ 
    & =  \left<  (Q^+ \circ F \circ Q^-)(\overline{x}^\circ)) ,   \overline{y}  - \overline{x}^\circ  \right>_{\overline{\sX}}.  
    \end{align}
    \label{eq: product-scalar-product-equiv-2}
\end{subequations}
Together the inequality (\ref{eq: product-scalar-product-equiv-2}) and the fact $\overline{y}\in\sA$ was arbitrarily chosen imply $\overline{x}^\circ \in \mathrm{VI}(Q^+\circ F\circ Q^-,\sA)$.
This completes the proof.
\end{proof}

\begin{algorithm*}[t]
\caption{ Nash Fixed Point Network\  (Minkowski Sum Contraints $\sC = \sC_1 + \cdots + \sC_K$)}
\label{alg: N-FPN-product}
\begin{algorithmic}[1]           
    \STATE{\begin{tabular}{p{0.47\textwidth}r}
     \hspace*{-8pt}  $\sN_\Theta(d):$
     & 
     $\vartriangleleft$ Input data is $d$
     \end{tabular}}    
    
    \STATE{\begin{tabular}{p{0.47\textwidth}r}
    \hspace*{2pt}{\bf $n \leftarrow 1$}
     & 
     $\vartriangleleft$  Initialize counter
     \end{tabular}}        
    
    \STATE{\begin{tabular}{p{0.47\textwidth}r}
    \hspace*{2pt}{\bf for   $k =1,2,\ldots,K$}
     & 
     \end{tabular}}    
     
    \STATE{\begin{tabular}{p{0.47\textwidth}r}
     \hspace*{10pt}$\overline{z}_k^1 \leftarrow \hat{z}$
     & 
     $\vartriangleleft$ Initialize iterates to $\hat{z}\in\sX$   
     \end{tabular}}

    \STATE{\begin{tabular}{p{0.47\textwidth}r}
     \hspace*{2pt}{\bf while}   $ \sum_{k=1}^K \|\overline{z}_k^{n} - \overline{z}_k^{n-1}\| > \varepsilon$ {\bf or} $n=1$
     & 
     $\vartriangleleft$ Loop until convergence at fixed point
     \end{tabular}}  \\[5pt]
    
    \STATE{\begin{tabular}{p{0.47\textwidth}r}
    \hspace*{10pt}{\bf for   $k = 1,2,\ldots,K $}
     & 
     $\vartriangleleft$  Loop over constraints $\sC_k^1$
     \end{tabular}}     
    \STATE{\begin{tabular}{p{0.47\textwidth}r}
     \hspace*{18pt} $\overline{x}_k^{n+1} \leftarrow P_{\sC_k^{1}}(\overline{z}_k^n)$
     & 
     $\vartriangleleft$ Project onto constraint set
     \end{tabular}}    \\[5pt]
     
    \STATE{\begin{tabular}{p{0.47\textwidth}r}
     \hspace*{10pt}$v^{n+1} \leftarrow  \sum_{k=1}^K \overline{x}^{n+1}_k$
     & 
     $\vartriangleleft$ Combine projections
     \end{tabular}}        \\[5pt]   
     
    \STATE{\begin{tabular}{p{0.47\textwidth}r}
    \hspace*{10pt}{\bf for  $k =1,2,\ldots,K$}
     & 
     $\vartriangleleft$  Loop over constraints $\sC_k^2$
     \end{tabular}}  \\[5pt]   

    \STATE{\begin{tabular}{p{0.47\textwidth}r}
     \hspace*{18pt}$\overline{y}_k^{n+1} \leftarrow P_{\sC_k^{2}}(2\overline{x}_k^{n+1}-\overline{z}_k^n-\alpha F_\Theta(v^{n+1};d))$
     & 
     $\vartriangleleft$ Block-wise project reflected gradients
     \end{tabular}}      \\[5pt]
     
    \STATE{\begin{tabular}{p{0.47\textwidth}r}
     \hspace*{18pt}$\overline{z}_k^{n+1} \leftarrow \overline{z}_k^n -\overline{x}_k^{n+1} + \overline{y}_k^{n+1}$
     & 
     $\vartriangleleft$ Apply block-wise updates
     \end{tabular}}      \\[5pt] 
     
    \STATE{\begin{tabular}{p{0.47\textwidth}r}
     \hspace*{10pt}$n\leftarrow n+1$
     & 
     $\vartriangleleft$ Increment counter
     \end{tabular}}        \\[5pt]

    \STATE{\begin{tabular}{p{0.47\textwidth}r}
     \hspace*{2pt}{\bf return} $v^n$
     & 
     $\vartriangleleft$ Output inference 
     \end{tabular}}   
\end{algorithmic}
\end{algorithm*}    

\section{Related Works}
\label{section: related-works}
There are two distinct learning problems for games. The first considers repeated rounds of the same game and operates from the agent's perspective. The agents are assumed to have imperfect knowledge of the game, and the goal is to learn the optimal strategy ({\em i.e.} the Nash equilibrium or a coarse correlated equilibrium), given only the cost incurred in each round. This   problem is not investigated in this work, and we refer the reader to \cite{hannan1957approximation,stoltz2007learning,sessa2020contextual} for further details.

\begin{figure}[t]
    \centering 
    \begin{tabular}{ccc}
    \includegraphics[width=0.3\textwidth]{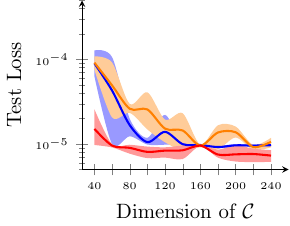}
    &
    \includegraphics[width=0.3\textwidth]{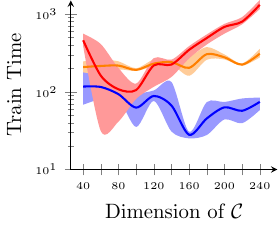} &
    \includegraphics[width=0.3\textwidth]{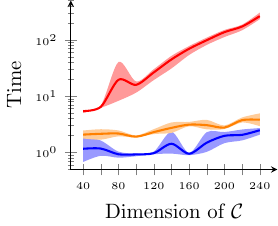}
    \\
    \end{tabular}
    \caption{\reedit{Final test loss ({\bf left}), total training time ({\bf center}) and mean training time per epoch ({\bf right}) for Payoff-Net (shown in \textcolor{red}{red}), a cocoercive N-FPN (shown in \textcolor{orange}{orange}) and an unconstrained N-FPN (shown in \textcolor{blue}{blue}). Each network is trained for 100 epochs or until a test loss less than $10^{-5}$ is achieved. The final test loss decreases as a function of $a$, which is expected since the number of parameters increases with $a$. Note that in this experiment the form of N-FPN without three-operator splitting (\ie \Cref{alg: N-FPN-Simple}) is used, and so the speed-up in train time observed is attributable to the fact that N-FPN uses fixed-point iteration for forward propagation and JFB for backward propagation, while Payoff-Net uses Newton's method on the forward pass and solves \eqref{eq: adjoint} on the backward pass.
    }}
    \label{fig: Generalized_RPS}
\end{figure}  

The second problem supposes historical observations of agents' behaviour are available to an external observer. \reedit{For example, \cite{ratliff2014social,konstantakopoulos2016smart,allen2021using} posit a simple functional form of the agents cost functions $u_k$ which depend linearly on a set of unknown parameters. Assuming a set of noisy observations of the equilibrium\footnote{some of the aforementioned work considers equilibria other than Nash, {\em e.g.} \cite{waugh2011computational} considers a correlated equilibrium while \cite{allen2021using} considers generalized Nash equilibria, leading to additional technical challenges} $x^{\star}$ is observed, a regression problem can be formulated and solved to obtain an estimate of these parameters. These works do not consider cost functions depending on the context $d$. A similar approach is pursued in \cite{waugh2011computational}, except instead of attempting to estimate the unknown parameters in the agent's costs functions, an equilibrium $x^{\circ}$ is predicted which explains the agent's behaviour for {\em all possible} values of these unknown parameters. Small-scale traffic routing problems are considered. \\ \cite{bertsimas2015data,zhang2016price,zhang2017data,zhang2018price} are important precursors to our work. Similar to us, they view the primary object of study as a variational inequality with unknown $F$. Given noisy observations of solutions to this variational inequality, \cite{bertsimas2015data} proposes a non-parametric, kernel based method for approximating $F$. This method is applied in \cite{zhang2016price,zhang2017data,zhang2018price} to non-contextual traffic routing problems on road networks discussed in Section~\ref{subsec: TrafficRouting}. While it is conceivable that this method could be extended to contextual games, to the best of the author's knowledge this has not yet been done.}

Several recent works \cite{ling2018game,ling2019large,li2020end} consider data consisting of pairs of contexts $d$ and equilibria $x_d^{\star}$ of the contextual game parameterized by $d$, and employ techniques from contemporary deep learning. Crucially, \cite{ling2018game} is the first paper to propose a differentiable game solver---which we refer to as Payoff-Net---allowing for end-to-end training of a neural network that predicts $x_d^{\star}$ given $d$. \reedit{Abstractly, the output of Payoff-Net is defined as the Nash equilibrium of the game
\begin{equation}
    \min_{x_1 \in \Delta_n}\min_{x_2 \in \Delta_n} x_1^{\top}B_{\Theta}(d)x_2 - \sum_j x_{1,j}\log(x_{1,j}) + \sum_j x_{2,j}\log(x_{2,j}),
    \label{eq:PayOff-Net}
\end{equation}
where $B_{\Theta}(d)$ is a neural network whose output is an antisymmetric $n\times n$ matrix, while $\Delta_n$ is the $n$-probability simplex (see \Cref{remark: QRE} for further discussion on the role of the entropic regularizers). The KKT conditions for \eqref{eq:PayOff-Net} are
\begin{equation}
    \begin{split}
        B_{\Theta}(d)x_2 + \log(x_1) + 1 + \mu 1 &= 0 \\
        B_{\Theta}(d)^{\top}x_1 - \log(x_2) - 1 + \nu1 &= 0 \\
        1^{\top}x_1 &= 1 \\
        1^{\top}x_2 &= 1
    \end{split}
    \label{eq:PayOff-Net_KKT}
\end{equation}
where $1$ (respectively $0$) represents the all-ones (respectively all-zeros) vector of appropriate dimension, and $\log$ is applied elementwise. The forward pass of Payoff-Net applies Newton's method to \eqref{eq:PayOff-Net_KKT}, at a cost of $\mathcal{O}(n^3)$ per iteration (see \cite{amos2017optnet} for further discussion on this complexity). Differentiating \eqref{eq:PayOff-Net_KKT} with respect to $B_{\Theta}$ yields a linear system which may be solved for $\dd{x^{\star}_d}{B_{\Theta}}$ at a cost of $\mathcal{O}(n^3)$. This is done on the backward pass of Payoff-Net. From $\dd{x^{\star}_d}{B_{\Theta}}$ one may compute $\dd{x^{\star}_d}{\Theta}$ via the chain rule. We highlight that, by construction, Payoff-Net can only be applied to two-player, zero-sum games with $\mathcal{C} = \Delta_n\times\Delta_n$.   
In \cite{ling2019large}, this approach was modified, leading to a faster backpropagation algorithm, but only for two-player, zero sum games with $\mathcal{C} = \Delta_n\times\Delta_n$ which admit a compact extensive form representation.} 
In \cite{li2020end} a differentiable variational inequality layer (VI-Layer), similar to \eqref{eq:Def_as_VI}, is proposed. Using the equivalence \eqref{eq: VI-Landweber}, they convert the problem of training this VI-Layer to that of tuning a parametrized operator $F_{\Theta}(\cdot;\cdot)$ such that
\begin{equation*}
     x_d^\star \approx P_{\sC}(x_d^\star - F_{\Theta}(x_d^\star;d)).
\end{equation*}
This idea is a significant step forward as it extends the approach of \cite{ling2018game} to unregularized games with arbitrary $\sC$ and an arbitrary number of players. \reedit{It also extends \cite{bertsimas2015data}, by connecting their approach with the techniques of deep learning.} However, as \cite{li2020end} does not use constraint decoupling (see  \Cref{thm:Use_DYS} and  \Cref{sec: Decoupling}) they are forced to use an iterative $\sO(\text{dim}(\sC)^3)$ algorithm  \cite{amos2017optnet} to compute $P_{\sC}$ (resp. $\mathrm{d}{P_{\sC}}/\mathrm{d}{z}$) in every forward (resp. backward) pass, \edit{as compared to the $\mathcal{O}(\text{dim}(\sC)^2)$ cost of N-FPN}. When $F_{\Theta}(\cdot;\cdot)$ is a multi-layer neural network, tuning $\Theta$ might require millions of forward and backward passes. Thus, their approach is impractical for games with even moderately large $\sC$ (see Section 3.3 of \cite{amos2017optnet}). \edit{Since the arXiv version of this work \cite{heaton2021learn} appeared, the use of N-FPNs for contextual traffic routing has been furthered by \cite{liu2023end}, where OD-pair (see \Cref{subsec: TrafficRouting}) specific contextual dependencies are considered.}


Our N-FPN architecture, particularly the use of operator splitting techniques, leverages insights from projection methods, which in Euclidean spaces date back to the 1930s \cite{cimmino1938cacolo,karczmarz1937angenaherte}. Projection methods are well-suited to large-scale problems as they are built from projections onto individual sets, which are often easy to compute; see \cite{censor2015projection,censor2012effectiveness} and the references therein. \edit{Finally, we note that the learning problem \eqref{eq:Empirical_Risk} is an example of a Mathematical Program with Equilibrium Constraints (MPEC) \cite{luo1996mathematical}. In this context, the difficulty of ``differentiating through'' the fixed point $z^{\circ}_d$ (see \eqref{eq: adjoint}) is well-known, and we refer the reader to \cite{li2023achieving} for further discussion on computing this derivative, as well as an alternative approach for doing so.}
 
\section{Numerical Examples}
We show the efficacy of N-FPNs on two classes of contextual games: matrix games and traffic routing.

\subsection{Contextual Matrix Games}
\label{sec:Experiment_RPS}
In \cite{ling2018game} the Payoff-Net architecture is used for a contextual ``rock-paper-scissors'' game. This is a (symmetric) matrix game where both players have action sets of dimension $3$. We extend this experiment to higher-dimensional action sets. Note \cite{ling2018game} consider {\em entropy-regularized} cost functions \footnote{equivalently: they determine the Quantal Response Equilibrium not the Nash Equilibrium, see \cref{remark: QRE}.}:
\begin{equation}
\begin{split}
    u_1(x;d) = x_2^{\top}B(d)x_1 + \sum_i x_{1,i}\log(x_{1,i}) \\
    u_2(x;d) = -x_2^{\top}B(d)x_1 + \sum_i x_{2,i}\log(x_{2,i}), 
\end{split}
\label{eq:matrix-game}
\end{equation}

\begin{wrapfigure}[14]{r}{0.5\textwidth}
    \centering 
    \includegraphics[width=0.48\textwidth]{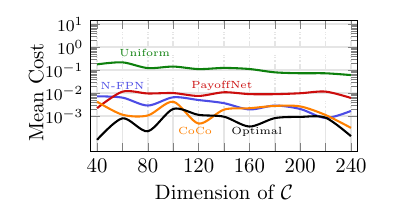} 
    \caption{Simulated play for matrix games of increasing size. Here ``N-FPN'' refers to the unconstrained variant, while ``CoCo'' refers to the cocoercive variant of N-FPN.} 
    \label{fig: RPS}
\end{wrapfigure}
\noindent for antisymmetric contextual cost matrix $B(d) \in \bbR^{a\times a}$, thus guaranteeing the game satisfies assumptions (A1)--(A5), particularly (A4). We do the same here. Each player's set of mixed strategies is the probability simplex $\Delta_a$ so $\sC = \Delta_a \times \Delta_a$. We vary $a$ \reedit{in multiples of $10$ from $20$ to $120$}. For each $a$ we generate 
a training data set $\{(d^i,x_{d^i}^\star)\}_{i=1}^{2000}$ and train a Payoff-Net, \reedit{an N-FPN constrained to be cocoercive using \eqref{eq:N-FPN-CoCo}} and an unconstrained N-FPN with comparable numbers of parameters for 100 epochs or until the test loss is below $10^{-5}$, whichever comes first. \reedit{See \Cref{app: Matrix Games} for further architectural details.} The results are presented in Figure~\ref{fig: Generalized_RPS}.\\
\reedit{Payoff-Net achieves the target test loss in much fewer epochs than (either version of) N-FPN. We attribute this to the use of Newton's method on the forward pass (which approximates the Nash equilibrium to higher precision) as well as the use of the true gradient on the backward pass. However, the time Payoff-Net requires to complete an epoch grows exponentially with the size of $\sC$ (see Figure~\ref{fig: Generalized_RPS}). Hence, for larger $\sC$ it is one to two orders of magnitude faster to train an N-FPN to the desired test loss.}

For illustration, we simulate play in the unregularized (\ie without the entropic term in \eqref{eq:matrix-game}) matrix game between two agents over a test set of contexts $d$. The first agent has full access to $B(d)$ and plays according to the computed Nash equilibrium. Four options are used for the second agent:
\begin{itemize}
    \item A N-FPN agent, who plays the strategy provided by a trained (unconstrained) N-FPN given $d$. 
    \item A Payoff-Net agent, who plays the strategy provided by a trained Payoff-Net given $d$.
    \item A data-agnostic agent, who plays the uniform strategy (\ie each action is selected with equal probability) regardless of $d$.
    \item An optimal agent, who has full access to $B(d)$ and plays according to the computed Nash equilibrium.
\end{itemize}
\reedit{We plot the absolute value of the mean cost, averaged over all $1000$ trials per $d$ and all $d$ for a given action set size $a$. The results are illustrated in  \Cref{fig: RPS}. As this game is zero-sum, the expected mean cost is zero. Over all $a$, N-FPN outperforms Payoff-Net. We attribute this to the fact that Payoff-Net explicitly incorporates the entropic regularizer into its architecture (see Section~\ref{section: related-works}), whereas the unconstrained N-FPN does not.}

\begin{figure}[t]
    \centering
    \subfloat[Rainy Day Prediction]{
    \includegraphics[width=3in]{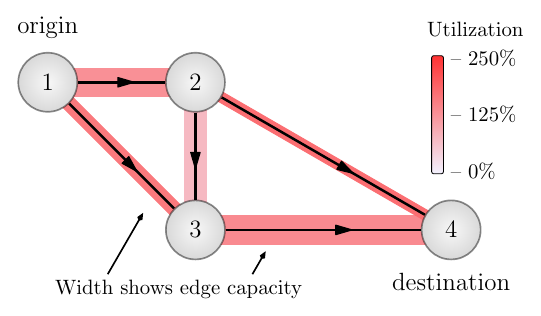}
    }
    \subfloat[Rainy Day True]{
    \includegraphics[width=3in]{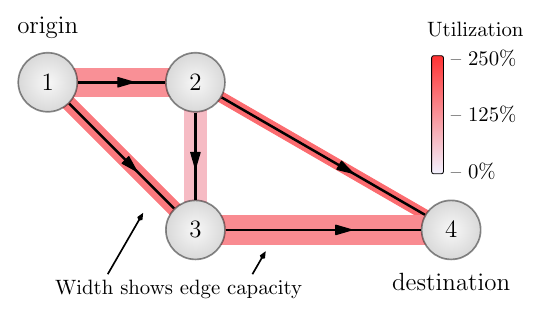}}
    
    \subfloat[Sunny Day Prediction]{
    \includegraphics[width=3in]{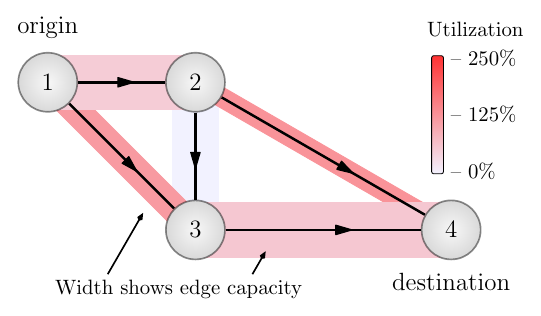}}
    \subfloat[Sunny Day True]{
    \includegraphics[width=3in]{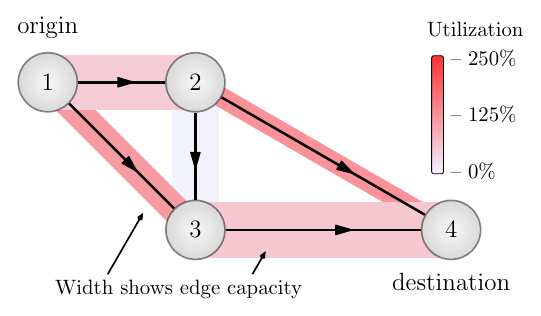}}

    \caption{(a): True traffic flow for ``rainy'' context. (b): Predicted traffic by $\sN_\Theta$ for ``rainy'' context. (c): True traffic flow for ``sunny'' context. (d): Predicted traffic by $\sN_\Theta$ for ``sunny'' context} 
    \label{fig:Braess}
\end{figure}
    
\subsection{Contextual Traffic Routing}
\label{subsec: TrafficRouting}
\paragraph{Setup} Consider a road network represented by a directed graph with vertices $V$ and arcs $E$. Let $N\in\mathbb{R}^{|V|\times |E|}$ denote the vertex-arc incidence matrix defined by
\begin{equation}
    N_{ij} \triangleq \left\{\begin{array}{cc} +1 & \text{ if } (i,j) \in E \\ -1 & \text{ if } (j,i) \in E \\ 0 & \text{ otherwise }\end{array}\right.
\end{equation}
For example, for the simple road network shown in  \Cref{fig:Braess} the incidence matrix is
\begin{equation}
    N = \begin{bmatrix} 
    -1 & 0 & -1 &  0 &  0 \\
    0 &  0 &  1 & -1 & -1 \\ 
    1 & -1 &  0 &  1 &  0 \\ 
    0 &  1 &  0 &  0 &  1 
    \end{bmatrix}.
    \label{eq:Braess_Incidence}
\end{equation} 
An origin-destination pair (OD-pair) is a triple $(v_1,v_2,q)$ with $v_{i}\in V$ and $q \in \bbR_{>0}$, encoding the constraint of routing  $q$ units of traffic from $v_1$ to $v_2$. Each OD-pair is encoded by  a vector $b \in \bbR^{|V|}$ with $b_{v_1} = -q$, $b_{v_2} = q$ and all other entries zero. A valid  {\em traffic flow} $x\in \bbR^{|E|}$ for an OD-pair has nonnegative entries  satisfying the flow equation $Nx = b$. The $e$-th entry $x_e$ represents the traffic density along the $e$-th arc. The flow equation ensures    the number of cars entering an intersection equals the number leaving, except for a net movement of $q$ units of traffic from $v_1$ to $v_2$. 
For $K$ OD-pairs, a valid traffic flow $x$ is the sum of  traffic flows for each OD-pair, which is in the Minkowski sum:
\begin{equation}
    \sC = \sum_{k=1}^K
    \sC_k \triangleq \underbrace{\{ x : Nx=b_k\}}_{\sC_k^1} \bigcap \underbrace{ \{ x: x\geq 0\}}_{\sC_k^2}, 
    \label{eq:C_for_Wardrop}
\end{equation} 
A {\em contextual travel time function} $t_{e}(x_e;d)$ is associated with each arc, where $d$ encodes contextual data. This function increases monotonically with $x_e$, reflecting the fact that increased congestion leads to longer travel times. The context $d$ encodes exogenous factors --- weather, construction and so on. 
Here the equilibrium of interest is, roughly speaking, a flow configuration $x_d^\star$ where the travel time between each OD-pair is as short as possible when taking into account congestion effects \cite{carlier2012continuous}. 
This is known as a \textit{Wardrop equilibrium} (also called the {\em user equilibrium}) \cite{wardrop1952road}, a special case of Nash equilibria where 
$
    F = [t_1(x_1;d)^\top \cdots t_{|E|}(x_{|E|};d)^\top]^\top.
$ In certain cases, a Wardrop equilibrium is the limit of a sequence of Nash equilibria as the number of drivers goes to infinity \cite{haurie1985relationship}. 

\paragraph{TRAFIX Scores} Accuracy of traffic routing predictions are measured by a {\em TRAFIX score}. This score forms an intuitive alternative to mean squared error. 
An error tolerance $\varepsilon > 0$ is chosen (\nb $\varepsilon= 5\times 10^{-3}$ in our experiments). For an estimate $x$ of $x^\star$, the TRAFIX score with parameter $\varepsilon$ is the percentage of edges for which $x$ has relative error (with  tolerance\footnote{The parameter $\tau$ is added to handle the case when the $e$-th component of $x^\star$ is zero, \ie $x_e^\star = 0$.} $\tau >0$) less than $\varepsilon$, \ie 
\begin{subequations}
    \begin{align}
    &\mbox{(rel. error of edge $e$)}  \triangleq  \dfrac{ |x_e-x^\star_e|}{|x^\star_e|+\tau } , \nonumber\\
    &\mbox{TRAFIX}(x,x^\star; \varepsilon,\tau)   \triangleq \dfrac{ {\mbox{(\# edges with rel. error $<\varepsilon$})}}{\mbox{(\#  edges)}}  \times 100\%. \nonumber
    \end{align}
\end{subequations}
Our plots and tables show the expected TRAFIX scores over the distributions of testing data. 

\begin{figure}[t]
    \centering 
    \begin{tabular}{cc}
    \includegraphics[width=.45\textwidth]{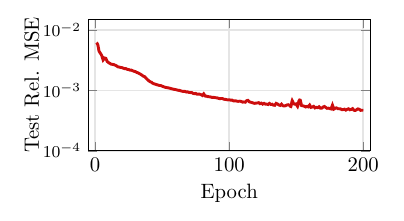} 
    &
    \includegraphics[width=.45\textwidth]{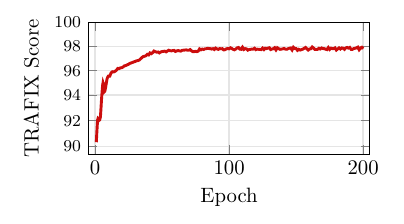} 
    \end{tabular}
    \caption{Plots for N-FPN performance on Eastern Massachusetts testing data. The first plot shows convergence of expected relative mean squared error on testing data after each training epoch and the second plot shows the expected TRAFIX score on testing data after each training epoch.}
    \label{fig: convergence_Easter-Mass}
\end{figure}

\def\ROWCOLOR{black!10!white}
\begin{table*}[htp]
    \centering 
    \small
    \begin{tabular}{c c c c}
        dataset & edges/nodes & OD-pairs & \# params 
        \\
        \toprule
        \rowcolor{\ROWCOLOR}
        Sioux Falls        
        & 76/24 
        & 528 
        & 46K 
        \\
        Eastern Mass.
        & 258/74 
        & 1113
        & 99K
        \\
        \rowcolor{\ROWCOLOR}
        Berlin-Friedrichshain     
        & 523/224
        & 506
        & 179K
        \\
        Berlin-Tiergarten
        & 766/361 
        & 644 
        & 253K 
        \\
        \rowcolor{\ROWCOLOR}
        Anaheim
        & 914/416
        & 1406
        & 307K
        \\
        \edit{
        Chicago-Sketch}
        &
        2950/933
        &
        93513
        &
        457K 
    \end{tabular}
    \vspace{0pt}
    \caption{Datasets used. First and second columns show the number of edges, nodes, and origin-destination pairs for corresponding dataset.
    Second column shows the number of tunable parameters.
    Further details may be found in  supplementary materials.
    }
    \label{tab: trafficNetResults}
\end{table*}

\begin{table}[htp]
    \centering 
    \small
    \begin{tabular}{c cc cc}
        & \multicolumn{2}{c}{MSE} & \multicolumn{2}{c}{TRAFIX}\\ 
        dataset & N-FPN &  feedforward &  N-FPN & feedforward
        \\
        \toprule
        \rowcolor{\ROWCOLOR}
        Sioux Falls
        &$1.9 \times 10^{-3}$ 
        & $5.4 \times 10^{-3} $
        & 94.42\% 
        & 70.16\%
        \\
        Eastern Mass.
        & $4.7 \times 10^{-4}$ 
        & $4.1 \times 10^{-3} $
        & 97.94\% 
        & 92.70\%
        \\
        \rowcolor{\ROWCOLOR}
        Berlin-Friedrichshain 
        & $5.3 \times 10^{-4}$
        & $9.3 \times 10^{-4} $
        & 97.42\% 
        & 97.94\%
        \\
        Berlin-Tiergarten 
        & $7.6 \times 10^{-4}$
        & $5.5 \times 10^{-4}$
        & 95.95\%
        & 97.03\%
        \\
        \rowcolor{\ROWCOLOR}
        Anaheim
        & $2.4 \times 10^{-3}$
        & $5.1 \times 10^{-2} $
        & 95.28\%
        & 58.57\%
        \\
        Chicago-Sketch
        & $2.5 \times 10^{-3}$
        & $3.1 \times 10^{-3}$
        & 98.81\%
        & 97.12\%
    \end{tabular}
    \caption{Results of Traffic Routing Experiments. \edit{To benchmark our results, we provide comparison with a traditional neural network architecture. To make a fair comparison, we use the same architecture for $F_\Theta$ in N-FPN and the feedforward neural network.}}
\end{table}

\paragraph{Datasets and Training} 
We are unaware of any prior datasets for contextual traffic routing, and so we construct our own. First, we construct a toy example based on the ``Braess paradox'' network studied in \cite{li2020end}, illustrated in \Cref{fig:Braess}. Here $d\in \mathbb{R}^5$; 
see supplementary materials for  further details.

We also constructed contextual traffic routing data sets based on road networks of real-world cities curated by the Transportation Networks for Research Project~\cite{traffixnet}. We did so by fixing a choice of $t_{e}(x;d)$ for each arc $e$, randomly generating a large set of contexts $d \in [0,1]^{10}$ and then, for each $d$, finding a solution $x_d^\circ \in$ \VI. \Cref{tab: trafficNetResults} shows a description of the traffic networks datasets, including the numbers of edges, nodes, and OD-pairs. 
Further details  are in the supplementary materials. 
We emphasize that for these contextual games the structure of $\sC$ is complex; it is a Minkowski sum of hundreds of high-dimensional polytopes (recall \Cref{eq:C_for_Wardrop}).  
We train an N-FPN using the constraint decoupling described in \Cref{sec: Decoupling} for forward propagation 
(see \Cref{alg: N-FPN-product}) 
to predict $x^{\star}_d$ from $d$ for each data set with architectures as described in
the appendix.
Additional training details  are in the appendix.
\edit{For comparison, we also train a traditional feedforward neural network. We use the same architecture used to parameterize the game gradient in N-FPN and use the same number of epochs during training. We perform a logarithmic search when tuning the learning rate.}

\paragraph{Results}  
As illustrated in \Cref{fig:Braess}, the N-FPN almost perfectly predicts the resulting Wardrop equilibrium given only the context $d$. The results for the real-world networks are shown in the final two columns of \Cref{tab: trafficNetResults}. The convergence during training of the relative MSE and TRAFIX score on the Eastern-Massachusetts testing dataset is shown in  \Cref{fig: convergence_Easter-Mass}. 
Additional plots can be found in  the supplementary materials.

\section{Conclusions} 
The fusion of big data and optimization algorithms offers potential for predicting equilibria in  systems with many interacting agents.
The proposed N-FPNs  form  a \textit{scalable} data-driven framework for efficiently predicting equilibria for such systems that can be modeled as contextual games. The N-FPN architecture yields equilibria outputs that satisfy constraints while also being trained end-to-end.  Moreover, the provided constraint decoupling schemes  enable  simple forward and backward propagation using explicit formulae for each projection.
The   efficacy of  N-FPNs is illustrated on large-scale traffic routing problems using a  contextual traffic routing benchmark dataset and  TRAFIX scoring system. Although we focus here on games, we note that N-FPNs are equally applicable to any system modeled using a variational inequality (or equivalently a linear complementarity problem), for example convex optimization \cite{amos2017optnet,wilder2019melding} or physical simulation \cite{de2018end}. Future work shall focus on end-to-end learning for these domains using N-FPN.

\section*{Acknowledgments}
HH, DM, SO, SWF and QL were supported by AFOSR
MURI \\ FA9550-18-1-0502 and ONR grants: N00014-18-
1-2527, N00014-20-1-2093, and N00014-20-1-2787. HH’s
work was also supported by the National Science Foundation (NSF) Graduate Research Fellowship under Grant No.
DGE-1650604. 
SWF was also supported in part by National Science Foundation award DMS-2309810 and DMS-2110745.
Any opinion, findings, and conclusions or
recommendations expressed in this material are those of the
authors and do not necessarily reflect the views of the NSF.

\newpage
\bibliographystyle{siamplain}
\bibliography{references}

\begin{thebibliography}{10}

\bibitem{agrawal2019differentiable}
{\sc A.~Agrawal, B.~Amos, S.~Barratt, S.~Boyd, S.~Diamond, and J.~Z. Kolter},
  {\em Differentiable convex optimization layers}, Advances in neural
  information processing systems, 32 (2019).

\bibitem{allen2021using}
{\sc S.~Allen, J.~P. Dickerson, and S.~A. Gabriel}, {\em Using inverse
  optimization to learn cost functions in generalized {N}ash games}, arXiv
  preprint arXiv:2102.12415,  (2021).

\bibitem{amos2017optnet}
{\sc B.~Amos and J.~Z. Kolter}, {\em Optnet: Differentiable optimization as a
  layer in neural networks}, in International Conference on Machine Learning,
  PMLR, 2017, pp.~136--145.

\bibitem{amos2017input}
{\sc B.~Amos, L.~Xu, and J.~Z. Kolter}, {\em Input convex neural networks}, in
  International Conference on Machine Learning, PMLR, 2017, pp.~146--155.

\bibitem{arrow1954existence}
{\sc K.~J. Arrow and G.~Debreu}, {\em Existence of an equilibrium for a
  competitive economy}, Econometrica: Journal of the Econometric Society,
  (1954), pp.~265--290.

\bibitem{azar2011soccer}
{\sc O.~H. Azar and M.~Bar-Eli}, {\em Do soccer players play the mixed-strategy
  nash equilibrium?}, Applied Economics, 43 (2011), pp.~3591--3601.

\bibitem{bai2019deep}
{\sc S.~Bai, J.~Z. Kolter, and V.~Koltun}, {\em Deep equilibrium models}, in
  Advances in Neural Information Processing Systems, 2019, pp.~690--701.

\bibitem{bai2020multiscale}
{\sc S.~Bai, V.~Koltun, and J.~Z. Kolter}, {\em Multiscale deep equilibrium
  models}, Advances in Neural Information Processing Systems, 33 (2020).

\bibitem{baillon1977quelques}
{\sc J.-B. Baillon and G.~Haddad}, {\em Quelques propri{\'e}t{\'e}s des
  op{\'e}rateurs angle-born{\'e}s et n-cycliquement monotones}, Israel Journal
  of Mathematics, 26 (1977), pp.~137--150.

\bibitem{bauschke2010baillon}
{\sc H.~H. Bauschke and P.~L. Combettes}, {\em The baillon-haddad theorem
  revisited}, Journal of Convex Analysis, 17 (2010), pp.~781--787.

\bibitem{bauschke2017convex}
{\sc H.~H. Bauschke, P.~L. Combettes, et~al.}, {\em Convex Analysis and
  Monotone Operator Theory in Hilbert Spaces}, Springer, 2nd~ed., 2017.

\bibitem{bertsimas2015data}
{\sc D.~Bertsimas, V.~Gupta, and I.~C. Paschalidis}, {\em Data-driven
  estimation in equilibrium using inverse optimization}, Mathematical
  Programming, 153 (2015), pp.~595--633.

\bibitem{bisong2019google}
{\sc E.~Bisong}, {\em Google colaboratory}, in Building Machine Learning and
  Deep Learning Models on Google Cloud Platform, Springer, 2019, pp.~59--64.

\bibitem{carlier2012continuous}
{\sc G.~Carlier and F.~Santambrogio}, {\em A continuous theory of traffic
  congestion and {W}ardrop equilibria}, Journal of Mathematical Sciences, 181
  (2012), pp.~792--804.

\bibitem{cegielski2012iterative}
{\sc A.~Cegielski}, {\em Iterative methods for fixed point problems in Hilbert
  spaces}, vol.~2057, Springer, Berlin, Germany, 2012.

\bibitem{censor2015projection}
{\sc Y.~Censor and A.~Cegielski}, {\em {Projection Methods: An Annotated
  Bibliography of Books and Reviews}}, Optimization, 64 (2015), pp.~2343--2358,
  \url{https://doi.org/10.1080/02331934.2014.957701}.

\bibitem{censor2012effectiveness}
{\sc Y.~Censor, W.~Chen, P.~L. Combettes, R.~Davidi, and G.~T. Herman}, {\em On
  the effectiveness of projection methods for convex feasibility problems with
  linear inequality constraints}, Computational Optimization and Applications,
  51 (2012), pp.~1065--1088.

\bibitem{chen2021learning}
{\sc T.~Chen, X.~Chen, W.~Chen, H.~Heaton, J.~Liu, Z.~Wang, and W.~Yin}, {\em
  Learning to optimize: A primer and a benchmark}, arXiv preprint
  arXiv:2103.12828,  (2021).

\bibitem{cimmino1938cacolo}
{\sc G.~Cimmino}, {\em Cacolo approssimato per le soluzioni dei systemi di
  equazioni lineari}, La Ricerca Scientifica (Roma), 1 (1938), pp.~326--333.

\bibitem{dafermos1988sensitivity}
{\sc S.~Dafermos}, {\em Sensitivity analysis in variational inequalities},
  Mathematics of Operations Research, 13 (1988), pp.~421--434.

\bibitem{davis2017three}
{\sc D.~Davis and W.~Yin}, {\em A three-operator splitting scheme and its
  optimization applications}, Set-valued and variational analysis, 25 (2017),
  pp.~829--858.

\bibitem{de2018end}
{\sc F.~de~Avila Belbute-Peres, K.~Smith, K.~Allen, J.~Tenenbaum, and J.~Z.
  Kolter}, {\em End-to-end differentiable physics for learning and control},
  Advances in neural information processing systems, 31 (2018), pp.~7178--7189.

\bibitem{facchinei2007finite}
{\sc F.~Facchinei and J.-S. Pang}, {\em Finite-dimensional variational
  inequalities and complementarity problems}, Springer Science \& Business
  Media, 2007.

\bibitem{fung2022jfb}
{\sc S.~W. Fung, H.~Heaton, Q.~Li, D.~McKenzie, S.~Osher, and W.~Yin}, {\em
  Jfb: {Jacobian-free Backpropagation for Implicit Networks}}, Proceedings of
  the AAAI Conference on Artificial Intelligence,  (2022).

\bibitem{geng2021training}
{\sc Z.~Geng, X.-Y. Zhang, S.~Bai, Y.~Wang, and Z.~Lin}, {\em On training
  implicit models}, Advances in Neural Information Processing Systems, 34
  (2021), pp.~24247--24260.

\bibitem{ghaoui2019implicit}
{\sc L.~E. Ghaoui, F.~Gu, B.~Travacca, A.~Askari, and A.~Y. Tsai}, {\em
  Implicit deep learning}, arXiv preprint arXiv:1908.06315,  (2019).

\bibitem{gilton2021deep}
{\sc D.~Gilton, G.~Ongie, and R.~Willett}, {\em Deep equilibrium architectures
  for inverse problems in imaging}, arXiv preprint arXiv:2102.07944,  (2021).

\bibitem{golub2013matrix}
{\sc G.~H. Golub and C.~F. Van~Loan}, {\em {Matrix Computations}}, Johns
  Hopkins University Press, 2013.

\bibitem{hannan1957approximation}
{\sc J.~Hannan}, {\em Approximation to {B}ayes risk in repeated play},
  Contributions to the Theory of Games, 21 (1957), p.~97.

\bibitem{haurie1985relationship}
{\sc A.~Haurie and P.~Marcotte}, {\em On the relationship between
  nash—cournot and wardrop equilibria}, Networks, 15 (1985), pp.~295--308.

\bibitem{heaton2022explainable}
{\sc H.~Heaton and S.~W. Fung}, {\em Explainable ai via learning to optimize},
  arXiv preprint arXiv:2204.14174,  (2022).

\bibitem{heaton2021feasibility}
{\sc H.~Heaton, S.~W. Fung, A.~Gibali, and W.~Yin}, {\em Feasibility-based
  fixed point networks}, arXiv preprint arXiv:2104.14090,  (2021).

\bibitem{heaton2022wasserstein}
{\sc H.~Heaton, S.~W. Fung, A.~T. Lin, S.~Osher, and W.~Yin}, {\em
  Wasserstein-based projections with applications to inverse problems}, SIAM
  Journal on Mathematics of Data Science, 4 (2022), pp.~581--603.

\bibitem{heaton2021learn}
{\sc H.~Heaton, D.~McKenzie, Q.~Li, S.~W. Fung, S.~Osher, and W.~Yin}, {\em
  Learn to predict equilibria via fixed point networks}, arXiv preprint
  arXiv:2106.00906,  (2021).

\bibitem{heaton2021learning}
{\sc H.~W. Heaton}, {\em Learning to Optimize with Guarantees}, PhD thesis,
  University of California, Los Angeles, 2021.

\bibitem{jahn2005system}
{\sc O.~Jahn, R.~H. M{\"o}hring, A.~S. Schulz, and N.~E. Stier-Moses}, {\em
  System-optimal routing of traffic flows with user constraints in networks
  with congestion}, Operations research, 53 (2005), pp.~600--616.

\bibitem{karczmarz1937angenaherte}
{\sc S.~Karczmarz}, {\em Angenaherte auflosung von systemen linearer
  glei-chungen}, Bull. Int. Acad. Pol. Sic. Let., Cl. Sci. Math. Nat.,  (1937),
  pp.~355--357.

\bibitem{kidger2020universal}
{\sc P.~Kidger and T.~Lyons}, {\em Universal approximation with deep narrow
  networks}, in Conference on learning theory, PMLR, 2020, pp.~2306--2327.

\bibitem{kingma2015adam}
{\sc D.~P. Kingma and J.~Ba}, {\em Adam: A method for stochastic optimization},
  in ICLR (Poster), 2015.

\bibitem{konstantakopoulos2016smart}
{\sc I.~C. Konstantakopoulos, L.~J. Ratliff, M.~Jin, C.~Spanos, and S.~S.
  Sastry}, {\em Smart building energy efficiency via social game: a robust
  utility learning framework for closing--the--loop}, in 2016 1st International
  Workshop on Science of Smart City Operations and Platforms Engineering
  (SCOPE) in partnership with Global City Teams Challenge (GCTC)(SCOPE-GCTC),
  IEEE, 2016, pp.~1--6.

\bibitem{kotary2021end}
{\sc J.~Kotary, F.~Fioretto, P.~Van~Hentenryck, and B.~Wilder}, {\em End-to-end
  constrained optimization learning: A survey}, arXiv preprint
  arXiv:2103.16378,  (2021).

\bibitem{koyama2022music}
{\sc Y.~Koyama, N.~Murata, S.~Uhlich, G.~Fabbro, S.~Takahashi, and
  Y.~Mitsufuji}, {\em Music source separation with deep equilibrium models}, in
  ICASSP 2022-2022 IEEE International Conference on Acoustics, Speech and
  Signal Processing (ICASSP), IEEE, 2022, pp.~296--300.

\bibitem{krantz2012implicit}
{\sc S.~G. Krantz and H.~R. Parks}, {\em The implicit function theorem:
  history, theory, and applications}, Springer Science \& Business Media, 2012.

\bibitem{krasnoselskii1955two}
{\sc M.~Krasnosel'ski\u{\i}}, {\em Two remarks about the method of successive
  approximations}, Uspekhi Mat. Nauk, 10 (1955), pp.~123--127.

\bibitem{leblanc1975efficient}
{\sc L.~J. LeBlanc, E.~K. Morlok, and W.~P. Pierskalla}, {\em An efficient
  approach to solving the road network equilibrium traffic assignment problem},
  Transportation research, 9 (1975), pp.~309--318.

\bibitem{li2023achieving}
{\sc J.~Li, J.~Yu, B.~Liu, Z.~Wang, and Y.~M. Nie}, {\em Achieving
  hierarchy-free approximation for bilevel programs with equilibrium
  constraints}, arXiv preprint arXiv:2302.09734,  (2023).

\bibitem{li2020end}
{\sc J.~Li, J.~Yu, Y.~Nie, and Z.~Wang}, {\em End-to-end learning and
  intervention in games}, Advances in Neural Information Processing Systems, 33
  (2020).

\bibitem{ling2018game}
{\sc C.~K. Ling, F.~Fang, and J.~Z. Kolter}, {\em What game are we playing?
  end-to-end learning in normal and extensive form games}, arXiv preprint
  arXiv:1805.02777,  (2018).

\bibitem{ling2019large}
{\sc C.~K. Ling, F.~Fang, and J.~Z. Kolter}, {\em Large scale learning of agent
  rationality in two-player zero-sum games}, in Proceedings of the AAAI
  Conference on Artificial Intelligence, vol.~33, 2019, pp.~6104--6111.

\bibitem{liu2022inducing}
{\sc B.~Liu, J.~Li, Z.~Yang, H.-T. Wai, M.~Hong, Y.~Nie, and Z.~Wang}, {\em
  Inducing equilibria via incentives: Simultaneous design-and-play ensures
  global convergence}, Advances in Neural Information Processing Systems, 35
  (2022), pp.~29001--29013.

\bibitem{liu2023end}
{\sc Z.~Liu, Y.~Yin, F.~Bai, and D.~K. Grimm}, {\em End-to-end learning of user
  equilibrium with implicit neural networks}, Transportation Research Part C:
  Emerging Technologies, 150 (2023), p.~104085.

\bibitem{luo1996mathematical}
{\sc Z.-Q. Luo, J.-S. Pang, and D.~Ralph}, {\em Mathematical programs with
  equilibrium constraints}, Cambridge University Press, 1996.

\bibitem{mann1953mean}
{\sc R.~Mann}, {\em Mean {Value} {Methods} in {Iteration}}, 4 (1953),
  pp.~506--510.

\bibitem{marcotte1995convergence}
{\sc P.~Marcotte and J.~H. Wu}, {\em On the convergence of projection methods:
  application to the decomposition of affine variational inequalities}, Journal
  of Optimization Theory and Applications, 85 (1995), pp.~347--362.

\bibitem{mckelvey1995quantal}
{\sc R.~D. McKelvey and T.~R. Palfrey}, {\em Quantal response equilibria for
  normal form games}, Games and economic behavior, 10 (1995), pp.~6--38.

\bibitem{mckenzie2023faster}
{\sc D.~McKenzie, S.~W. Fung, and H.~Heaton}, {\em Faster predict-and-optimize
  with three-operator splitting}, arXiv preprint arXiv:2301.13395,  (2023).

\bibitem{mertikopoulos2016learning}
{\sc P.~Mertikopoulos and W.~H. Sandholm}, {\em Learning in games via
  reinforcement and regularization}, Mathematics of Operations Research, 41
  (2016), pp.~1297--1324.

\bibitem{miyato2018spectral}
{\sc T.~Miyato, T.~Kataoka, M.~Koyama, and Y.~Yoshida}, {\em Spectral
  normalization for generative adversarial networks}, arXiv preprint
  arXiv:1802.05957,  (2018).

\bibitem{nash1950equilibrium}
{\sc J.~F. Nash}, {\em Equilibrium points in n-person games}, Proceedings of
  the national academy of sciences, 36 (1950), pp.~48--49.

\bibitem{pedregosa2018adaptive}
{\sc F.~Pedregosa and G.~Gidel}, {\em Adaptive three operator splitting}, in
  International Conference on Machine Learning, PMLR, 2018, pp.~4085--4094.

\bibitem{pesquet2021learning}
{\sc J.-C. Pesquet, A.~Repetti, M.~Terris, and Y.~Wiaux}, {\em Learning
  maximally monotone operators for image recovery}, SIAM Journal on Imaging
  Sciences, 14 (2021), pp.~1206--1237.

\bibitem{ramzi2022shine}
{\sc Z.~Ramzi, F.~Mannel, S.~Bai, J.-L. Starck, P.~Ciuciu, and T.~Moreau}, {\em
  Shine: Sharing the inverse estimate from the forward pass for bi-level
  optimization and implicit models}, in ICLR 2022-International Conference on
  Learning Representations, 2022.

\bibitem{ratliff2014social}
{\sc L.~J. Ratliff, M.~Jin, I.~C. Konstantakopoulos, C.~Spanos, and S.~S.
  Sastry}, {\em Social game for building energy efficiency: Incentive design},
  in 2014 52nd Annual Allerton Conference on Communication, Control, and
  Computing (Allerton), IEEE, 2014, pp.~1011--1018.

\bibitem{rockafellar1970convex}
{\sc R.~T. Rockafellar}, {\em Convex analysis}, vol.~36, Princeton university
  press, 1970.

\bibitem{romano2017little}
{\sc Y.~Romano, M.~Elad, and P.~Milanfar}, {\em The little engine that could:
  Regularization by denoising (red)}, SIAM Journal on Imaging Sciences, 10
  (2017), pp.~1804--1844.

\bibitem{rosen1965existence}
{\sc J.~B. Rosen}, {\em Existence and uniqueness of equilibrium points for
  concave n-person games}, Econometrica: Journal of the Econometric Society,
  (1965), pp.~520--534.

\bibitem{roughgarden2007routing}
{\sc T.~Roughgarden}, {\em Routing games}, Algorithmic game theory, 18 (2007),
  pp.~459--484.

\bibitem{ryu2022large}
{\sc E.~Ryu and W.~Yin}, {\em Large-Scale Convex Optimization: Algorithm
  Designs via Monotone Operators}, Cambridge University Press, Cambridge,
  England, 2022.

\bibitem{salimans2021should}
{\sc T.~Salimans and J.~Ho}, {\em Should ebms model the energy or the score?},
  in Energy Based Models Workshop-ICLR 2021, 2021.

\bibitem{sessa2020contextual}
{\sc P.~G. Sessa, I.~Bogunovic, A.~Krause, and M.~Kamgarpour}, {\em Contextual
  games: Multi-agent learning with side information}, Advances in Neural
  Information Processing Systems, 33 (2020).

\bibitem{traffixnet}
{\sc B.~Stabler, H.~Bar-Gera, and E.~Sall}, {\em Transportation networks for
  research.}
\newblock \url{https://github.com/bstabler/TransportationNetworks.}, 2016.
\newblock Accessed: 2021-05-24.

\bibitem{stoltz2007learning}
{\sc G.~Stoltz and G.~Lugosi}, {\em Learning correlated equilibria in games
  with compact sets of strategies}, Games and Economic Behavior, 59 (2007),
  pp.~187--208.

\bibitem{vapnik1999nature}
{\sc V.~Vapnik}, {\em The nature of statistical learning theory}, Springer
  science \& business media, 1999.

\bibitem{wardrop1952road}
{\sc J.~G. Wardrop}, {\em Some theoretical aspects of road traffic research.},
  Proceedings of the institution of civil engineers, 1 (1952), pp.~325--362.

\bibitem{waugh2011computational}
{\sc K.~Waugh, B.~D. Ziebart, and J.~A. Bagnell}, {\em Computational
  rationalization: the inverse equilibrium problem}, in Proceedings of the 28th
  International Conference on International Conference on Machine Learning,
  2011, pp.~1169--1176.

\bibitem{wilder2019melding}
{\sc B.~Wilder, B.~Dilkina, and M.~Tambe}, {\em Melding the data-decisions
  pipeline: Decision-focused learning for combinatorial optimization}, in
  Proceedings of the AAAI Conference on Artificial Intelligence, vol.~33, 2019,
  pp.~1658--1665.

\bibitem{winston2020monotone}
{\sc E.~Winston and J.~Z. Kolter}, {\em Monotone operator equilibrium
  networks}, in Advances in Neural Information Processing Systems,
  H.~Larochelle, M.~Ranzato, R.~Hadsell, M.~F. Balcan, and H.~Lin, eds.,
  vol.~33, Curran Associates, Inc., 2020, pp.~10718--10728,
  \url{https://proceedings.neurips.cc/paper/2020/file/798d1c2813cbdf8bcdb388db0e32d496-Paper.pdf}.

\bibitem{zhang2017data}
{\sc J.~Zhang and I.~C. Paschalidis}, {\em Data-driven estimation of travel
  latency cost functions via inverse optimization in multi-class transportation
  networks}, in 2017 IEEE 56th Annual Conference on Decision and Control (CDC),
  IEEE, 2017, pp.~6295--6300.

\bibitem{zhang2016price}
{\sc J.~Zhang, S.~Pourazarm, C.~G. Cassandras, and I.~C. Paschalidis}, {\em The
  price of anarchy in transportation networks by estimating user cost functions
  from actual traffic data}, in 2016 IEEE 55th Conference on Decision and
  Control (CDC), IEEE, 2016, pp.~789--794.

\bibitem{zhang2018price}
{\sc J.~Zhang, S.~Pourazarm, C.~G. Cassandras, and I.~C. Paschalidis}, {\em The
  price of anarchy in transportation networks: Data-driven evaluation and
  reduction strategies}, Proceedings of the IEEE, 106 (2018), pp.~538--553.

\bibitem{zhangmultiset}
{\sc Y.~Zhang, D.~W. Zhang, S.~Lacoste-Julien, G.~J. Burghouts, and C.~G.
  Snoek}, {\em Multiset-equivariant set prediction with approximate implicit
  differentiation}, in International Conference on Learning Representations.

\end{thebibliography}

\appendix

\maketitle
\section{Proofs of Theorems}
\label{app: universal_approx_proof}
We provide the proofs of several theorems omitted from the main text.
For ease of reference, we restate each theorem before its proof. \\

\noindent Theorem 4.1 ({\sc Universal Approximation}).
\textit{\thmUniversalApproximation}

\begin{proof}
Let  $\varepsilon > 0$ be given. 
Denote the map $d \mapsto x_d^{\star}$ by $\mathcal{L}$, {\em i.e.} $\mathcal{L}(d) \triangleq x_d^{\star}$.  By  \Cref{thm:Lipschitz_Equib}, $\mathcal{L}$ is well-defined and Lipschitz continuous. Combined with the compactness of $\sD$ via (A5), this implies, by standard universal approximation properties of neural networks \cite{kidger2020universal},   there exists a continuous neural network $G_{\Theta}:\sD\rightarrow\sX$ such that
\begin{equation}
    \max_{{d}\in\sD} \|\mathcal{L}({d}) -  G_{\Theta}({d})\|_2 \leq \frac{\varepsilon}{2}. 
    \label{eq: universal-approx-proof-01}
\end{equation}
Next fix $\alpha > 0$ and define the operator $F_\Theta\colon \sX\times\sD\rightarrow\sX$ by
\begin{equation}
    F_\Theta(x;d) \triangleq \dfrac{x - G_\Theta(d) }{\alpha}.
     \label{eq: universal-approx-proof-02}
\end{equation}
and recall the corresponding N-FPN is defined as
\begin{equation}
    \sN_\Theta(d) \triangleq \mathrm{VI}(F_\Theta(\cdot;d). \sC).
\end{equation}
Note $F_\Theta$ is continuous by the continuity of $G_\Theta$, and so the VI and fixed point equivalence \eqref{eq: VI-Landweber} implies, for any $\zeta \in\sD$,
\begin{equation}
    \sN_{\Theta}(\zeta) = P_{\sC}(x_\zeta^\circ - \alpha F_{\Theta}(x_\zeta^\circ;d))
    = P_{\sC}(G_\Theta(\zeta)).
     \label{eq: universal-approx-proof-03}
\end{equation}
By definition of the projection $P_{\sC}$, 
\begin{equation}
    \|P_{\sC}(G_\Theta(\zeta)) - G_\Theta(\zeta)\|_2 = \min_{x\in\sC} \|x-  G_\Theta(\zeta)\|_2, 
     \label{eq: universal-approx-proof-04}
\end{equation}
which implies, since $\mathcal{L}(\zeta) \in \sC$,
\begin{equation}
  \| P_{\sC}(G_{\Theta}(\zeta)) -  G_{\Theta}(\zeta)\|_2 \leq  \|\mathcal{L}(\zeta) - G_{\Theta}(\zeta)\|_2 .
   \label{eq: universal-approx-proof-05}
\end{equation}
Together with the triangle inequality, \eqref{eq: universal-approx-proof-01} and \eqref{eq: universal-approx-proof-05} yield 
\begin{subequations}
    \begin{align}
        \|x_\zeta^{\star} - \sN_{\Theta}(\zeta)\|_2 
        & = \|\mathcal{L}(\zeta) - P_{\sC}(G_\Theta(\zeta))\|_2 \\
        & \leq   \|\mathcal{L}(\zeta) - {G}_{\Theta}(\zeta)\|_2 + \|   {G}_{\Theta}(\zeta)  - P_{\sC}({G}_{\Theta}(\zeta)) \|_2\\ 
        & \leq 2 \|\mathcal{L}(\zeta) - {G}_{\Theta}(\zeta)\|_2 = \varepsilon         
    \end{align}\label{eq: universal-approx-proof-06}\end{subequations}Since \eqref{eq: universal-approx-proof-06} holds for arbitrarily chosen $\zeta\in\sD$, we deduce that $\displaystyle \max_{d \in \sD}\|x_{d}^{\star} - \sN_{\Theta}(d)\|_2 \leq \varepsilon$ holds for the provided $\varepsilon$. As $\varepsilon > 0$ was also arbitrarily chosen, the result follows.
\end{proof}

\noindent The next result verifies the DYS scheme yields a convergent sequence.\\ 

\noindent {Theorem 4.3.} 
\textit{\thmDysConvergence} \\ 

\begin{proof}
    We proceed in the following manner.
    First we show the sequence $\{x^k\}$ converges to the desired limit (Step 1) and the residual $\|x^k - x^{k-1}\|$ drops below $\epsilon$ after $\sO(\epsilon^{-2})$ iterations (Step 2). 
    Then we show the per-iteration complexity is $\sO(\mbox{dim}(\sC)^2)$ (Step 3). Multiplying the bounds in Steps 2 and 3 yields the desired result. 
    
    \textit{Step 1.} Because all sets considered in this work are closed and convex and nonempty, the projection operators $P_{\sC^1}$ and $P_{\sC^2}$ are averaged \cite[Theorem 2.2.21]{cegielski2012iterative}. 
    Combined with the fact that $F_\Theta$ is $\alpha$-cocoercive, while $\gamma = \alpha$ the operator  $T$ is $\frac{2}{3}$-averaged \cite[Proposition 2.1]{davis2017three}.
    The classic result of Krasnosel'ski\u{\i} \cite{krasnoselskii1955two} and Mann \cite{mann1953mean}   asserts if, given any $z^1$, a sequence $\{z^k\}$ is generated using updates of the form $z^{k+1} = T(z^k)$ for an averaged operator $T$, then $\{z^k\}$ converges to a fixed point of $T$. Thus, $z^k \rightarrow z^\circ \in \Omega_d \triangleq \{ z :  z = T_\Theta(z;d)\}\neq \varnothing$.
    Because the projection operator is 1-Lipschitz, it necessarily follows that $\{P_{\sC^1}(z^k)\}$ converges to $P_{\sC^1}(z^\circ)$. By our VI equivalence lemma, $P_{\sC^1}(z^\circ) \in \mathrm{VI}(F(\cdot;d), \sC)$.
    \textit{Step 2.}
    Define $\mbox{dist}(z, \Omega_d) \triangleq \inf\{\|z-\overline{z}\| : \overline{z}\in \Omega_d \}$.
    By Theorem 1 in \cite{ryu2022large} and the $\frac{2}{3}$-averagedness of $T_\Theta$,
    \begin{equation}
        \|z^{k} - z^{k-1}\|^2
        \leq \dfrac{\frac{2}{3}}{k(1-\frac{2}{3})} \cdot \mbox{dist}(z^1, \Omega_d)^2 
        = \dfrac{2}{k} \cdot \mbox{dist}(z^1, \Omega_d)^2,
        \quad \mbox{for all}\ k\geq 2.
    \end{equation}
    Since projection operators are $1$-Lipschitz \cite{bauschke2017convex}, it follows that
    \begin{equation}
        \|x^k - x^{k-1}\|^2
        \leq \|z^{k} - z^{k-1}\|^2 
        \leq \dfrac{2}{k} \cdot \mbox{dist}(z^1, \Omega_d)^2 ,
        \quad \mbox{for all}\ k\geq 2.
    \end{equation}
    Thus, 
    \begin{equation} 
        \|x^k - x^{k-1}\| \leq \epsilon,
        \quad \mbox{for all}\ 
        k \geq  \max\left\lbrace 2,\ \dfrac{2\cdot \mbox{dist}(z^1,\Omega_d)^2}{\epsilon^2}\right\rbrace .
    \end{equation}

    \textit{Step 3.} For per-iteration complexity, the projection taking the form of a ReLU has computational cost $\sO(\mbox{dim}(\sC))$ as each element of $z^k$ has an element-wise max operation applied. The cost to apply the matrix multiplication for the affine projection is $\sO(\mbox{dim}(\sC)^2)$ \cite{golub2013matrix}.\footnote{The calculations for updating $z^k$ include the computation of $x^k$, and so we ignore costs to compute $x^k$.}   
\end{proof}

\newpage
\section{Intersections of Constraints}

For completeness, we also consider constraints $\sC$ that may be expressed as the intersection of several sets, \ie $\sC = \sC_1\cap \sC_2 \cdots \cap \sC_K$. Let $\overline{\sX}$, $\left<\cdot,\cdot\right>_{\overline{\sX}}$, $Q^{+}$ and $Q^{-}$ be as in  \Cref{sec: Minkowski Sum}. Next define\footnote{Note $\sA$ in  \Cref{sec: Minkowski Sum} is the same as $\sB^{1}$ in (\ref{eq: B-def}), \ie $\sB^1 = \sA$.}
\begin{equation}
\begin{aligned}
    \sB^{1} & \triangleq \sC_1\times \cdots \times \sC_K,
    \\
     \sB^{2} & \triangleq Q^{+}(\sX) = \{\overline{x} \in \overline{\sX}: \ x_1=\cdots=x_K\},
     \\
     \sB &\triangleq \sB^{1}\cap \sB^{2},
    \label{eq: B-def}
\end{aligned}
\end{equation}
 Note   $Q^{-}(\sB) = \sC$. The logic is now the same as before; rephrase Algorithm 1 using $\sB^{i}$ in place of $\sC^{i}$. The projection $P_{\sB^{1}}$ can be computed component-wise via
\begin{equation}
    P_{\sB^{1}}(\overline{x}) =   \left(P_{\sC_1}(\overline{x}_1), \ldots, P_{\sC_K}(\overline{x}_K)\right),
\end{equation}
and $P_{\sB^{2}}(\overline{x})$ has a simple closed form given in the following lemma.

\begin{lemma}
With notation as above, $P_{\sB^{2}}(\overline{x}) = Q^{+}\left(\frac{1}{K}\sum_{k=1}^{K}\overline{x}_k\right)$.
\end{lemma}

\begin{proof}
By the definition of a projection and the norm on $\overline{\sX}$,
\begin{subequations} 
\begin{align} 
P_{\sB^{2}}(\overline{x}) 
& \triangleq \argmin_{\overline{z}\in \sB^{2}}\| \overline{z}-\overline{x}\|_{\overline{\sX}}^2 = \argmin_{\overline{z}\in \sB^{2}}\sum_{k=1}^{K}\|  \overline{z}_k-\overline{x}_k\|^2 = Q^{+}(z^{\#}), 
\end{align}
\label{eq: projection-equality}
\end{subequations}
where $z^{\#} = \argmin_{z\in \sX} \sum_{k=1}^{K}\|z - \overline{x}_k\|^2,$
so $z^\#$ satisfies the following optimality condition
\begin{align*}
    0 &= \dd{}{z}\left[ \sum_{k=1}^K \|z-\overline{x}_k\|^2 \right]_{z=z^\#} = \sum_{k=1}^K 2(z^\#-\overline{x}_k) = 2K\left( z^\# - \dfrac{1}{K}\sum_{k=1}^K \overline{x}_k\right).
\end{align*}
This implies
\begin{equation}
    z^\# = \dfrac{1}{K}\sum_{k=1}^K \overline{x}_k.
    \label{eq: z-optimal}
\end{equation}
Together (\ref{eq: projection-equality}) and (\ref{eq: z-optimal}) yield the result, completing the proof.
\end{proof}

For an operator $F\colon\sX\rightarrow\sX$, we define a corresponding product space operator $\overline{F}\colon \overline{\sX}\rightarrow\overline{\sX}$ via
\begin{equation}
    \overline{F}(\overline{x}) \triangleq (F(\overline{x}_1), \ldots, F(\overline{x}_K).
\end{equation}
This definition enables us to show a direct equivalence between a VI in the original space $\sX$ and the product space $\overline{\sX}$. That is, we complete the analysis in the following lemmas by showing the solution set of an appropriate VI in the product space coincides with that of the original VI.

\begin{proposition} 
    If  $F\colon \sX\rightarrow\sX$ is $\alpha$-cocoercive, then $\overline{F}\colon\overline{\sX}\rightarrow\overline{\sX}$ is $\alpha$-cocoercive.
\label{prop: Decoupling Minkowski sum}
\end{proposition}
\begin{proof}
    Fix any $\overline{x},\overline{y}\in \overline{\sX}$.
    Then observe
    \begin{subequations}
        \begin{align}
        \left< \overline{F}(\overline{x})-\overline{F}(\overline{y}), \overline{x}-\overline{y}\right>_{\overline{\sX}} &= \sum_{k=1}^K \left< F(\overline{x}_k)-F(\overline{y}_k), \overline{x}_k - \overline{y}_k\right>\\
        &\geq \alpha \sum_{k=1}^K \|F(\overline{x}_k)-F(\overline{y}_k)\|^2\\
        &= \alpha \|\overline{F}(\overline{x})-\overline{F}(\overline{y})\|_{\overline{\sX}}^2.
        \end{align}
        \label{eq: F-product-cocoerive}
    \end{subequations}
    Because (\ref{eq: F-product-cocoerive}) holds for arbitrarily chosen $\overline{x},\overline{y}\in\overline{\sX}$, we conclude $\overline{F}$ is $\alpha$-cocoercive.
\end{proof}

\begin{lemma}\label{lemma:VI_equivalence_2}
    For $\alpha$-cocoercive $F\colon \sX\rightarrow\sX$,
    $ x^\circ \in \mathrm{VI}\left(F,\sC\right)$
    iff
    $Q^+({x}^\circ) \in \mathrm{VI}\left(\overline{F}, \sB\right)$.
\end{lemma}
\begin{proof}
    Fix any $x^\circ \in \mathrm{VI}(F,\sC)$ and set $\overline{x}^\circ = Q^+(x^\circ)$.
    An elementary proof shows $Q^+ :\sC\rightarrow\sB$ is a bijection.
    Together with the fact $x^\circ$ is a VI solution, this implies
    \begin{subequations}
    \begin{align}
       & 0 \leq K\left< F(x^\circ), y-x^\circ\right>, \ \ \mbox{for all $y\in\sC$}
       \nonumber\\
        \ \iff \ 
        & 0 \leq \sum_{k=1}^K \left< F(x^\circ), \overline{y}_k - x^\circ\right>, \ \  \mbox{for all $\overline{y}\in\sB$} \\
        \ \iff \ & 0 \leq \sum_{k=1}^K \left< F(\overline{x}^\circ_k), \overline{y}_k - \overline{x}^\circ_k\right>, \ \ \mbox{for all $\overline{y}\in\sB$}\\
        \ \iff \ & 0 \leq  \left<\overline{F}(\overline{x}^\circ), \overline{y} - \overline{x}^\circ\right>_{\overline{\sX}}, \ \ \mbox{for all $\overline{y}\in\sB$}.
    \end{align}
    \label{eq: VI-product-equivalence}
    \end{subequations}
    By the transitive property, the first and final expressions in (\ref{eq: VI-product-equivalence}) are equivalent, and we are done.
\end{proof}

\subsection{Projections onto  Intersections of Hyperplanes}
\label{app: Projection onto an Affine Hyperplane}
Consider the set $\sC  \triangleq\{ x : Nx=b\} \subseteq \sX$, and note $\sC $ is closed and convex so the projection operator onto $\sC$ is well-defined and given by
\begin{equation}
\begin{aligned}
P_{\sC}(z) 
&=  \argmin_{x\in\sC} \frac{1}{2}\|x-z\|^2
\\
&= \argmin_{x\in\sX} \frac{1}{2}\|x-z\|^2\ \ \mathrm{s.t.}\ \ Nx=b.
\label{eq: projection-hyperplane-intersection-definition}
\end{aligned}
\end{equation}
For completeness we express (and prove) a projection formula for $\sC$ via the following lemma.
\begin{lemma}
    For nonempty $\sC \triangleq \{ x : Nx=b\} $, the projection $P_{\sC}$ is given by
    \begin{equation}
        P_{\sC}(z) = z - N^\dag (Nz-b),
    \end{equation}
    where $N^\dag \triangleq U{\Sigma}^{-1} V^\top$ and $U\Sigma V$ is the compact singular value decomposition of $N$ such that $U$ and $V$ have orthonormal columns and ${\Sigma}$ is invertible.
\end{lemma}
\begin{proof}
 Referring to (\ref{eq: projection-hyperplane-intersection-definition}), we see the associated Lagrangian is given by
\begin{equation}
    \sL(x,\lambda) \triangleq \dfrac{1}{2}\|x-z\|^2 + \left< \lambda, Nx-b\right>.
\end{equation}
The optimizer $x^\# \triangleq P_{\sC}(z)$  satisfies the   optimality condition $0 = \nabla  \sL(x^\#, \lambda^\#)$ for some $\lambda^\#$, which can be expanded as 
\begin{subequations}
\begin{align}
    0 &= \nabla_x \Big[ \sL(x,\lambda)\Big]_{(x,\lambda)=(x^\#,\lambda^\#)}  \nonumber
    \\
    &= x^\# -z  + N^\top \lambda^\#, \label{eq: opt1-Projection}\\
    0& = \nabla_\lambda \Big[ \sL(x,\lambda)\Big]_{(x,\lambda)=(x^\#,\lambda^\#)}  \nonumber \\
    &= N  x^\# - b. \label{eq: opt2-Projection}
\end{align}
\end{subequations}
We claim it suffices to choose
\begin{equation}
    \lambda^\# = (U\Sigma^{-2} U^\top) (Nz-b).
\end{equation}
By (\ref{eq: opt1-Projection}), this choice yields
\begin{subequations}
\begin{align}
    x^\# 
    &= z - N^\top \lambda^\#\\
    &= z - N^\top ( U\Sigma^{-2} U^\top) (Nz-b)\\
    &= z - (V\Sigma U^\top) (U\Sigma^{-2} U^\top)(Nz-b)\\
    &= z - (V \Sigma^{-1} U^\top)(Nz-b) \\
    &= z - N^\dag (Nz-b).
\end{align}
\end{subequations}
To prove this formula for $x^\#$ gives the projection, it suffices to show the remaining condition $Nx^\# =b$ is satisfied. Decomposing $N$ into its singular value decomposition, observe
\begin{subequations}
    \begin{align}
    Nx^\#
    &=  N(z - (V \Sigma^{-1} U^\top)(Nz-b)) \\
    &= Nz - (U\Sigma V^\top)(V\Sigma^{-1}U^\top)(Nz-b) \\
    &= Nz - (U\Sigma V^\top)(V\Sigma^{-1}U^\top)(U\Sigma V^\top z-b) \\
    &= Nz - U \Sigma V^\top z + U U^\top b \\
    &= U U^\top b.
    \end{align}    \label{eq: Nx-UUtop-b}\end{subequations}
The range of $N$ is contained in the subspace spanned by the orthonormal columns of $U$, \ie $\mbox{range}(N) \subseteq \mbox{span}(u^1,\ldots,u^r)$, where $u^i$ is the $i$-th column of $U$ and $r$ is the rank of $N$. 
Because $\sC$ is nonempty, $b \in \mbox{range}(N)$ and it follows that there exists scalars $\alpha_1,\ldots,\alpha_r$ such that 
\begin{equation}
    b = \sum_{i=1}^r \alpha_i u^i.
\end{equation}
Through direct substitution, we deduce
    \begin{equation}
    \begin{aligned}
        U U^\top b 
        &= U U^\top \sum_{i=1}^r \alpha_i u^i = U \left( \sum_{i,j=1}^r \alpha_i \left<u^j,u^i\right>\right) = U \sum_{i=1}^r \alpha_i e^i = \sum_{i=1}^r \alpha_i u^i = b.
    \end{aligned}
    \label{eq: UUtop-b}
\end{equation}
Thus, (\ref{eq: Nx-UUtop-b}) and (\ref{eq: UUtop-b}) together show the final optimality condition is satisfied,   proving the claim.  
\end{proof}

\begin{remark}
    In our traffic routing experiments, we use  the built-in SVD function in Pytorch,   threshold  the tiny singular values to be zero, and invert  the remaining entries.
\end{remark}

\section{Experimental Supplementary Material}
\label{sec:Experiment_Appendix}

\subsection{Matrix Games}
\label{app: Matrix Games}
Our Payoff-Net architecture is based upon the architecture described in \cite{ling2018game}\footnote{and downloaded from \url{https://github.com/lingchunkai/payoff_learning}}, but we make several modifications, which we now describe. First, we update their implementation of a differentiable game solver so as to be compatible with the current Pytorch autograd syntax. Second, we modify their code to enable Payoff-Net to handle matrix games of arbitrary size (not just $3$ dimensional action sets). Finally, we use two, instead of one, fully connected layers with ReLU activation to map context $d$ to payoff matrix $P$, which is then provided to the differentiable game solver. If $a$ is as in Section~\ref{sec:Experiment_RPS} then Payoff-Net has $9 + \frac{3}{2}\left(a^2 - a\right)$ tunable parameters.

The (unconstrained) N-FPN architecture is as described in \cref{alg: N-FPN-Simple} with $F_{\Theta}$ consisting of two fully connected. More specifically, 
\begin{equation*}
    F_{\Theta}(x;d) = x + W_2\left(x + \sigma\left(W_1d\right)\right)
\end{equation*}
where $W_1$ and $W_2$ are matrices of tunable parameters. Each N-FPN has $6a + 4a^2$ tunable parameters. \reedit{The cocoercive N-FPN architecture is also as described in \cref{alg: N-FPN-Simple}, except with 
\begin{equation*}
    F_{\Theta}(x;d) = \alpha x +W_2^{\top}W_2\left(x + \sigma(W_1d)\right) + (W_3 - W_3^{\top})x 
\end{equation*}
where $W_1, W_2,$ and $W_3$ are matrices of tunable parameters and $\alpha$ is the desired cocoercivity parameter. In our experiments we took $\alpha=0.5$. We also used spectral normalization, specifically the PyTorch utility {\tt nn.utility.spectral\_norm} applied to $A$ and $B$ to ensure $F_{\Theta}(\cdot;d)$ is Lipschitz continuous. This variant of N-FPN has $6a + 8a^2$ tunable parameters. }\\ 
For training all three networks we use Adam with a decrease-on-plateaus step-size scheduler. We tuned the initial step-sizes by performing a logarithmic sweep over $\{10^{-5},\ldots, 10^{0}\}$ and found that a starting step-size of $10^{-1}$ works bet for Payoff-Net while a starting step-size of $10^{-3}$ is best for (both versions of) N-FPN. All code is available \href{https://github.com/DanielMckenzie/Nash_FPNs}{online}. 

\subsection{Toy Traffic Routing Model}
\label{app: toy_traffic_routing_model}
We consider the traffic network shown in  \Cref{fig:Braess} (with incidence matrix given in \eqref{eq:Braess_Incidence}) and a single OD pair: $(v_1,v_4,1)$. We use the contextual travel-time functions\footnote{The form of this function is motivated by the well-known Bureau of Public Roads (BPR) function} 
\begin{equation}
    t_e(x_e ; d)
    \triangleq f_e \cdot \left( 1 + \left[ \dfrac{x_e}{c(d)_e}\right]^4\right), 
\end{equation}
where $f = (1, 2, \sqrt{2}, \sqrt{3}, 1)$   and 
\begin{equation}
    c(d) \triangleq \tilde{c}\odot \left(1 + P_{[-\varepsilon,\varepsilon]}(Wd) \right),
\end{equation}
for $\varepsilon = 0.4$ and   $\tilde{c} = (0.4, 0.8, 0.8, 0.6, 0.3)$, and $\odot$ denoting element-wise ({\em i.e} Hadamard) product. The matrix $W$ is constructed by sampling the entries of the first column uniformly and i.i.d. on $(-10,0]$, and sampling the entries of the remaining columns uniformly and i.i.d on $[0,1)$. This form of $W$ implies that $d_1 >0$ decreases the capacity of each road segment, albeit by varying amounts. Thus $d_1$ could be interpreted as, for example, inches of rainfall. We use this interpretation to generate  \Cref{fig: game-toy-diagram,fig:Braess}; taking any $d$ with $d_1$ large  corresponds to a rainy day while if $d_1$ small it can be interpreted as a sunny day. 
We generate training data by sampling $d$ i.i.d and uniformly from $[0,0.25]^5$  and then solving for $x^{\star}_d \in \mathrm{VI}(F,\sC)$ using    \Cref{alg: N-FPN-abstract} with
\begin{equation}
\begin{aligned}
    F(x;d) &\triangleq \left[ t_1(x_1;d), \ldots, t_5(x_5;d)\right]^\top,
    \\
\sC^{1} &\triangleq \{x: Nx = b\}, 
\\
b  &\triangleq [-1, \ 0,\  0,\ 1,\ 0 ]^\top,
\\
\sC^2 &=  \bbR_{\geq 0}^5
\end{aligned}
\end{equation}
 The projection onto $\sC^2$ is given by a component-wise ReLU and the projection onto $\sC^2$ is given in  \Cref{app: Projection onto an Affine Hyperplane}.

 \subsection{Real-World Traffic Routing Model}
 \label{app: real_traffic_routing_model}
 Similarly to  \Cref{app: toy_traffic_routing_model}, we consider a traffic network for the real data described in  \Cref{tab: trafficNetResults}.
 For each dataset, we obtain the OD pairs $b_k$, the free-flow time $f_e$, the incidence matrix $N$, and the capacity values $\tilde{c}$ on each edge from the Transportation Networks website~\cite{traffixnet}. 
 To generate the data, we use the contextual travel-time function 
 \begin{equation}
    t_e(x_e ; d)
    \triangleq f_e \cdot \left( 1 + 0.5 \left[ \dfrac{x_e}{c(d)_e}\right]^4\right), 
\end{equation}
where we contextualize the capacities with 
\begin{equation}
    c(d) \triangleq \tilde{c}\odot \left(1 + P_{[-\varepsilon,\varepsilon]}(Wd) \right).
\end{equation}
Here, we set $\epsilon = 0.8$ for the Anaheim dataset and $\epsilon = 0.3$ for the remaining four datasets.
We choose $\epsilon$ for the Anaheim dataset as we found the resulting actions $x^\star_d$ were too similar for $\epsilon = 0.3$ (making it too easy to train an operator fitting this dataset).
Similarly to the toy traffic problem, the matrix $W$ is constructed by sampling the entries of the first column uniformly and i.i.d. on $(-10,0]$, and sampling the entries of the remaining columns uniformly and i.i.d on $[0,1)$.
Since we have multiple OD pairs, the constraints are given by the Minkowski sum of polyhedral sets. Thus, we generate the 5500 training data pairs $(d, x^\star_d)$ using  \Cref{alg: N-FPN-product}.

\subsection{Network Architecture for Traffic Routing}
\label{app: network_architecture}
We describe the architectures used to generate  \Cref{tab: trafficNetResults}.
We use fully-connected layers to parameterize $F_\theta$. 
We have an opening layer, which maps from context (in our experiments, the context dimension is 10) to a 100-dimensional latent space. 
In the latent space, we use either one or two hidden layers (depending on the dataset) with 100-dimensional inputs and ouputs. 
The last layer maps from the hidden dimension to the action space, \ie, number of edges.
Since the number of edges vary per dataset, the number of tunable parameters also vary. 
Finally, we use a maximum depth of 50 iterations in our N-FPN architecture with a stopping tolerance of $\epsilon = 10^{-4}$.
\subsection{Training Details}
\label{app: training_setup}
As described in  \Cref{subsec: TrafficRouting}, we generate 5000 training samples and 500 testing samples for all datasets. 
For all datasets, we use batch size of $500$ and Adam~\cite{kingma2015adam} with constant learning rates and 200 epochs. The learning rates are chosen to be $5 \times 10^{-5}$ for Berlin-Tiergarten and $10^{-3}$ for the remaining datasets. 
All networks are trained using Google Colaboratory~\cite{bisong2019google}.

\section{Data Provenance}
\label{sec: Data Provenance}
For the Rock, Paper, Scissors experiment, we generated our own data following the experimental set-up described in \cite{ling2018game}. For the toy traffic routing problem, we also our own data, using the same traffic network as \cite{li2020end} but modifying their experiment so as to make road capacities contextual. 
The Sioux Falls, Berlin-Tiergarten and Berlin Friedrichshain and Eastern Massachussetts and datasets are from \cite{leblanc1975efficient,jahn2005system} and \cite{zhang2016price} respectively. The Anaheim dataset was provided by Jeff Ban and Ray Jayakrishnan and was originally hosted at \url{https://www.bgu.ac.il/~bargera/tntp/}. All datasets were downloaded from \cite{traffixnet} and are used under the ``academic use only'' license described therein.

\section{Additional Plots}
\label{app: additional_plots}
See Figures \ref{fig: Sioux_Falls}--\ref{fig: Berlin_Tiergarten}.

\begin{figure*}[t]
    \centering 
    \includegraphics{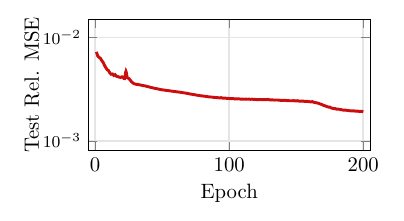} 
    \includegraphics{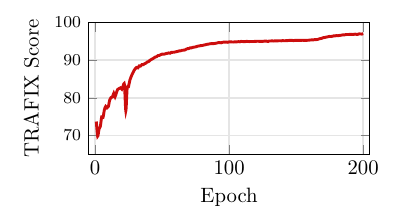} 
    \caption{Plots for N-FPN performance on Sioux Falls testing data. The left plot shows   convergence of expected mean squared error on testing data after each training epoch and the right shows the expected TRAFIX score on testing data after each training epoch.}
    \label{fig: Sioux_Falls}
\end{figure*}

\begin{figure*}[t]
    \centering 
        \includegraphics{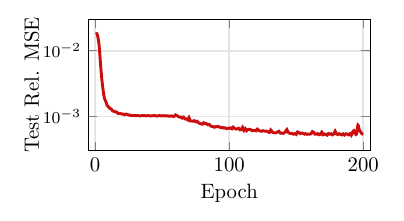} 
        \includegraphics {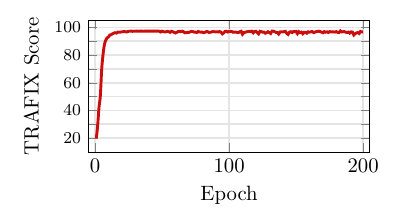} 
    \caption{Plots for N-FPN performance on Berlin Friedrichshain testing data. The left plot shows   convergence of expected relative mean squared error on testing data after each training epoch and the right shows the expected TRAFIX score on testing data after each training epoch.}
    \label{fig: Berlin_Friedrichshain}
\end{figure*}

\begin{figure*}[t]
    \centering 
        \includegraphics{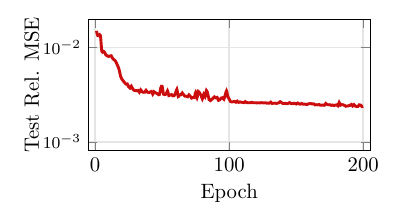} 
        \includegraphics{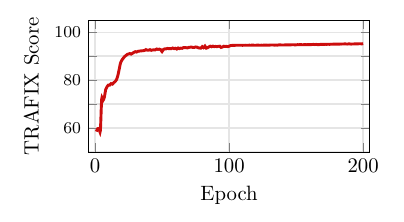} 
    \caption{Plots for N-FPN performance on Anaheim testing data. The left plot shows   convergence of expected relative mean squared error on testing data after each training epoch and the right shows the expected TRAFIX score on testing data after each training epoch.}
    \label{fig: Anaheim}
\end{figure*}

\begin{figure*}[t]
    \centering
    \includegraphics{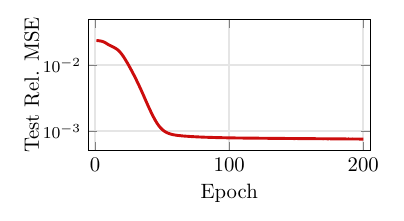}
    \includegraphics{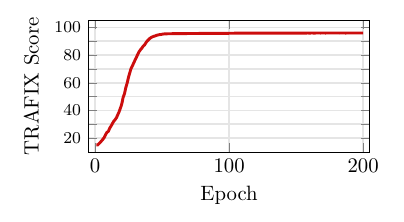}
    \caption{Plots for N-FPN performance on Berlin Tiergarten testing data. The left plot shows   convergence of expected relative mean squared error on testing data after each training epoch and the right shows the expected TRAFIX score on testing data after each training epoch.}
    \label{fig: Berlin_Tiergarten}
\end{figure*}

\section{Testing the importance of splitting}
The splitting architecture, not just JFB, is crucial to the success of N-FPN. To demonstrate this, we compare forward propagation and back propagation using two different architectures: one with splitting and one without splitting. We do this for the Sioux Falls traffic routing experiment. We time forward propagation and back propagation for increasing numbers of origin-destination pairs for a fixed batch size of 50 samples. \\
The method using splitting is N-FPN as presented in our paper. The method without splitting uses a projected-gradient descent-based approach:
    \begin{equation}
        x^{k+1} = P_{\mathcal{C}}(x^k - \alpha F_{\Theta}(x_k;d))
    \end{equation}
    where $P_{\mathcal{C}}(\cdot)$ is computed  using {\tt cvxpy-layers}~\cite{agrawal2019differentiable}. Recall that $\mathcal{C}$ is the Minkowski sum of $k$ polytopes, where $k$ is the number of OD pairs. So implementing $P_{\mathcal{C}}(\cdot)$ takes some care. \\
    In this scheme the JFB pseudogradient is (using the notation of Section 4.2)
    \begin{equation}
        p_{\Theta} = \frac{d\ell}{dx}\frac{dP_{\mathcal{C}}}{dx}\left(I - \frac{dF_{\Theta}}{dx}\right)
    \end{equation}
    where we compute $dP_{\sC}/dx$ using {\tt cvxpy-layers}. Note that $P_{\mathcal{C}}(\cdot)$ does not have a closed form.\\
    The tests are performed on an A100 GPU. As observed, for a \emph{small} batch size of 50 samples, the forward prop takes about \emph{10 minutes} when using an architecture without splitting. \emph{The Sioux Falls dataset has 128 OD pairs} and is the smallest of the traffic dataset. Thus training with an architecture without splitting, \emph{even with JFB} becomes intractable once all 128 OD pairs are included.
    \begin{table}
        \begin{center}
        \begin{tabular}{|c|c|c|c|}
            \hline
            origin-destination pairs & 2 & 5 & 10 
            \\
            \hline
            forward prop (no splitting) &  8.6 &  35.9 & 580.9
            \\
            forward prop (splitting) & 9e-3 & 1e-2 & 1e-2
            \\
            backprop (no splitting) & 1.2 & 1.5 & 1.7 
            \\
            backprop (splitting) & 4e-3 & 4e-3 & 4e-3 
            \\
            \hline
        \end{tabular}
        \end{center}
        \caption{Compute time (in seconds) for the forward propagation and backpropagation of the operator splitting-based network and a network without splitting. The network without constraint decoupling uses {\tt cvxpylayers}~\cite{agrawal2019differentiable} to solve the projection problem.}
    \end{table}

\end{document}